\documentclass{article}
\usepackage[table]{xcolor}

\PassOptionsToPackage{numbers,sort&compress}{natbib}
\usepackage[preprint]{neurips_2021}

\usepackage[utf8]{inputenc} %
\usepackage[T1]{fontenc}    %
\usepackage{hyperref}       %
\usepackage{url}            %
\usepackage{booktabs}       %
\usepackage{amsfonts}       %
\usepackage{nicefrac}       %
\usepackage{microtype}      %
\usepackage{xcolor}         %
\usepackage{graphicx}
\usepackage{subfigure}
\usepackage{xspace}
\usepackage{multirow}
\usepackage{glossaries}
\usepackage[normalem]{ulem}
\usepackage{wrapfig}
\usepackage{blindtext}
\usepackage[multiple]{footmisc}
\usepackage{dsfont}
\usepackage{caption}
\usepackage{amsthm}
\usepackage{xfrac}
\usepackage{listings}

\usepackage{amsmath,amsfonts,bm}

\def\eqref#1{equation~\ref{#1}}

\def\1{\bm{1}}
\newcommand{\realdata}{\mathcal{D}}
\newcommand{\generateddata}{\hat{\mathcal{D}}}
\newcommand{\train}{\mathcal{D_{\mathrm{train}}}}

\newcommand{\test}{\mathcal{D_{\mathrm{test}}}}

\def\vtheta{{\bm{\theta}}}
\def\vdelta{{\bm{\delta}}}

\def\vx{{\bm{x}}}

\DeclareMathAlphabet{\mathsfit}{\encodingdefault}{\sfdefault}{m}{sl}
\SetMathAlphabet{\mathsfit}{bold}{\encodingdefault}{\sfdefault}{bx}{n}

\def\gA{{\mathcal{A}}}

\def\gD{{\mathcal{D}}}

\def\gS{{\mathcal{S}}}

\def\gV{{\mathcal{V}}}
\def\gW{{\mathcal{W}}}
\def\gX{{\mathcal{X}}}
\def\gY{{\mathcal{Y}}}

\def\sA{{\mathbb{A}}}

\newcommand{\E}{\mathbb{E}}

\newcommand{\R}{\mathbb{R}}

\DeclareMathOperator*{\argmin}{arg\,min}

\DeclareMathOperator*{\maximize}{max}

\usepackage{algorithm}
\usepackage{algorithmicx}
\usepackage{algpseudocode}

\newacronym{pgd}{PGD}{Projected Gradient Descent}
\newacronym{bim}{BIM}{Basic Iterative Method}
\newacronym{fgsm}{FGSM}{Fast Gradient Sign Method}
\newacronym{wrn}{\textsc{Wrn}}{Wide ResNet}
\newacronym{sgd}{SGD}{Stochastic Gradient Descent}
\newacronym{ddpm}{DDPM}{Denoising Diffusion Probabilistic Model}
\newacronym{vdvae}{VDVAE}{Very Deep Variational Auto-Encoder}
\newacronym{fid}{FID}{Frechet Inception Distance}
\newacronym{is}{IS}{Inception Score}
\newacronym{gan}{GAN}{Generative Adversarial Network}
\newacronym{vae}{VAE}{Variational AutoEncoder}

\newcommand{\cifar}{\textsc{Cifar-10}\xspace}

\newcommand{\cifarh}{\textsc{Cifar-100}\xspace}
\newcommand{\tinyimages}{\textsc{80M-Ti}\xspace}
\newcommand{\imagenet}{\textsc{ImageNet}\xspace}
\newcommand{\tinyimagenet}{\textsc{TinyImageNet}\xspace}

\newcommand{\svhn}{\textsc{Svhn}\xspace}
\newcommand{\linf}{\ensuremath{\ell_\infty}\xspace}
\newcommand{\ltwo}{\ensuremath{\ell_2}\xspace}
\newcommand{\lp}{\ensuremath{\ell_p}\xspace}
\newcommand{\autoattack}{\textsc{AutoAttack}\xspace}
\newcommand{\autopgd}{\textsc{AutoPgd}\xspace}
\newcommand{\multitargeted}{\textsc{MultiTargeted}\xspace}

\newcommand{\pgd}[1]{\textsc{Pgd}\textsuperscript{$#1$}\xspace}
\newcommand{\xent}{l_{\textrm{ce}}}
\newcommand{\wrn}{\gls*{wrn}\xspace}

\def\fnr{{f_\textrm{NR}}}  %
\def\fs{{f^\star}}  %

\definecolor{TartOrange}{HTML}{ff2e35}
\definecolor{Orange}{HTML}{ff7825}
\definecolor{Mango}{HTML}{ffc013}
\definecolor{AppleGreen}{HTML}{7cb81b}
\definecolor{Blue}{HTML}{1173b0}
\definecolor{BdazzledBlue}{HTML}{2e58a5}
\definecolor{Purple}{HTML}{5b3590}
\definecolor{Sunglow}{HTML}{FFCA3A}

\definecolor{codegreen}{rgb}{0,0.6,0}
\definecolor{codegray}{rgb}{0.5,0.5,0.5}
\definecolor{codepurple}{rgb}{0.58,0,0.82}
\definecolor{backcolour}{rgb}{0.95,0.95,0.92}
\lstdefinestyle{pythonstyle}{
    backgroundcolor=\color{backcolour},
    commentstyle=\color{codegreen},
    keywordstyle=\color{magenta},
    numberstyle=\tiny\color{codegray},
    stringstyle=\color{codegreen},
    basicstyle=\ttfamily\tiny,
    breakatwhitespace=false,         
    breaklines=true,
    captionpos=b,
    keepspaces=true,
    numbers=left,
    numbersep=5pt,
    showspaces=false,
    showstringspaces=false,
    showtabs=false,
    tabsize=2
}
\definecolor{header}{gray}{0.9}
\definecolor{subheader}{rgb}{0.63, 0.79, 0.95}
\newcommand{\Tstrut}{\rule{0pt}{2.6ex}}
\newcommand{\Bstrut}{\rule[-0.9ex]{0pt}{0pt}}
\newcommand{\TBstrut}{\Tstrut\Bstrut}

\newcommand{\squishlist}{
   \begin{list}{$\bullet$}
    { \setlength{\itemsep}{0pt}      \setlength{\parsep}{3pt}
      \setlength{\topsep}{3pt}       \setlength{\partopsep}{0pt}
      \setlength{\leftmargin}{1.5em} \setlength{\labelwidth}{1em}
      \setlength{\labelsep}{0.5em} } }

\newcommand{\squishend}{
    \end{list}  }
    
\newtheorem{condition}{Condition}
\newtheorem{proposition}{Proposition}

\usepackage{tikz}
\usetikzlibrary{backgrounds}
\makeatletter
\tikzset{%
  fancy quotes/.style={
    text width=\fq@width pt,
    align=justify,
    inner sep=1em,
    anchor=north west,
    minimum width=\linewidth,
  },
  fancy quotes width/.initial={.8\linewidth},
  fancy quotes marks/.style={
    scale=8,
    text=white,
    inner sep=0pt,
  },
  fancy quotes opening/.style={
    fancy quotes marks,
  },
  fancy quotes closing/.style={
    fancy quotes marks,
  },
  fancy quotes background/.style={
    show background rectangle,
    inner frame xsep=0pt,
    background rectangle/.style={
      fill=gray!25,
      rounded corners,
    },
  }
}

\makeatother

\title{Improving Robustness using Generated Data}

\author{Sven Gowal*, Sylvestre-Alvise Rebuffi*, Olivia Wiles, \\
\textbf{Florian Stimberg}, \textbf{Dan Calian} {\normalfont and} {\bf Timothy Mann} \\
DeepMind, London\\
\texttt{\{sgowal,sylvestre\}@deepmind.com}
}

\begin{document}

\maketitle

\begin{abstract}
Recent work argues that robust training requires substantially larger datasets than those required for standard classification.
On \cifar and \cifarh, this translates into a sizable robust-accuracy gap between models trained solely on data from the original training set and those trained with additional data extracted from the ``80 Million Tiny Images'' dataset (\tinyimages).
In this paper, we explore how generative models trained solely on the original training set can be leveraged to artificially increase the size of the original training set and improve adversarial robustness to \lp norm-bounded perturbations.
We identify the sufficient conditions under which incorporating additional generated data can improve robustness, and demonstrate that it is possible to significantly reduce the robust-accuracy gap to models trained with additional real data.
Surprisingly, we show that even the addition of non-realistic random data (generated by Gaussian sampling) can improve robustness.
We evaluate our approach on \cifar, \cifarh, \svhn and \tinyimagenet against \linf and \ltwo norm-bounded perturbations of size $\epsilon = 8/255$ and $\epsilon = 128/255$, respectively.
We show large absolute improvements in robust accuracy compared to previous state-of-the-art methods.
Against \linf norm-bounded perturbations of size $\epsilon = 8/255$, our models achieve 66.10\% and 33.49\% robust accuracy on \cifar and \cifarh, respectively (improving upon the state-of-the-art by +8.96\% and +3.29\%).
Against \ltwo norm-bounded perturbations of size $\epsilon = 128/255$, our model achieves 78.31\% on \cifar (+3.81\%).
These results beat most prior works that use external data.
\end{abstract}

\section{Introduction}

Neural networks are being deployed in a wide variety of applications ranging from ranking content on the web~\citep{covington_deep_2016} to autonomous driving~\citep{bojarski_end_2016} via medical diagnostics~\citep{fauw_clinically_2018}.
It has become increasingly important to ensure that deployed models are robust and generalize to various input perturbations.
Unfortunately, the addition of imperceptible adversarial perturbations can cause neural networks to make incorrect predictions \citep{carlini_adversarial_2017,carlini_towards_2017,goodfellow_explaining_2014,kurakin_adversarial_2016,szegedy_intriguing_2013}.
There has been a lot of work on understanding and generating adversarial perturbations \citep{szegedy_intriguing_2013,biggio2013evasion,carlini_towards_2017,athalye_synthesizing_2017}, and on building defenses that are robust to such perturbations \citep{goodfellow_explaining_2014,madry_towards_2017,zhang_theoretically_2019,rice_overfitting_2020}.
We note that while robustness and invariance to input perturbations is crucial to the deployment of machine learning models in various applications, it can also have broader negative impacts to society such as hindering privacy~\cite{song2019privacy} or increasing bias~\cite{tramer_fundamental_2020}.

\begin{figure}[t]
\centering
\begin{minipage}{.52\textwidth}
  \centering
  \includegraphics[width=\textwidth]{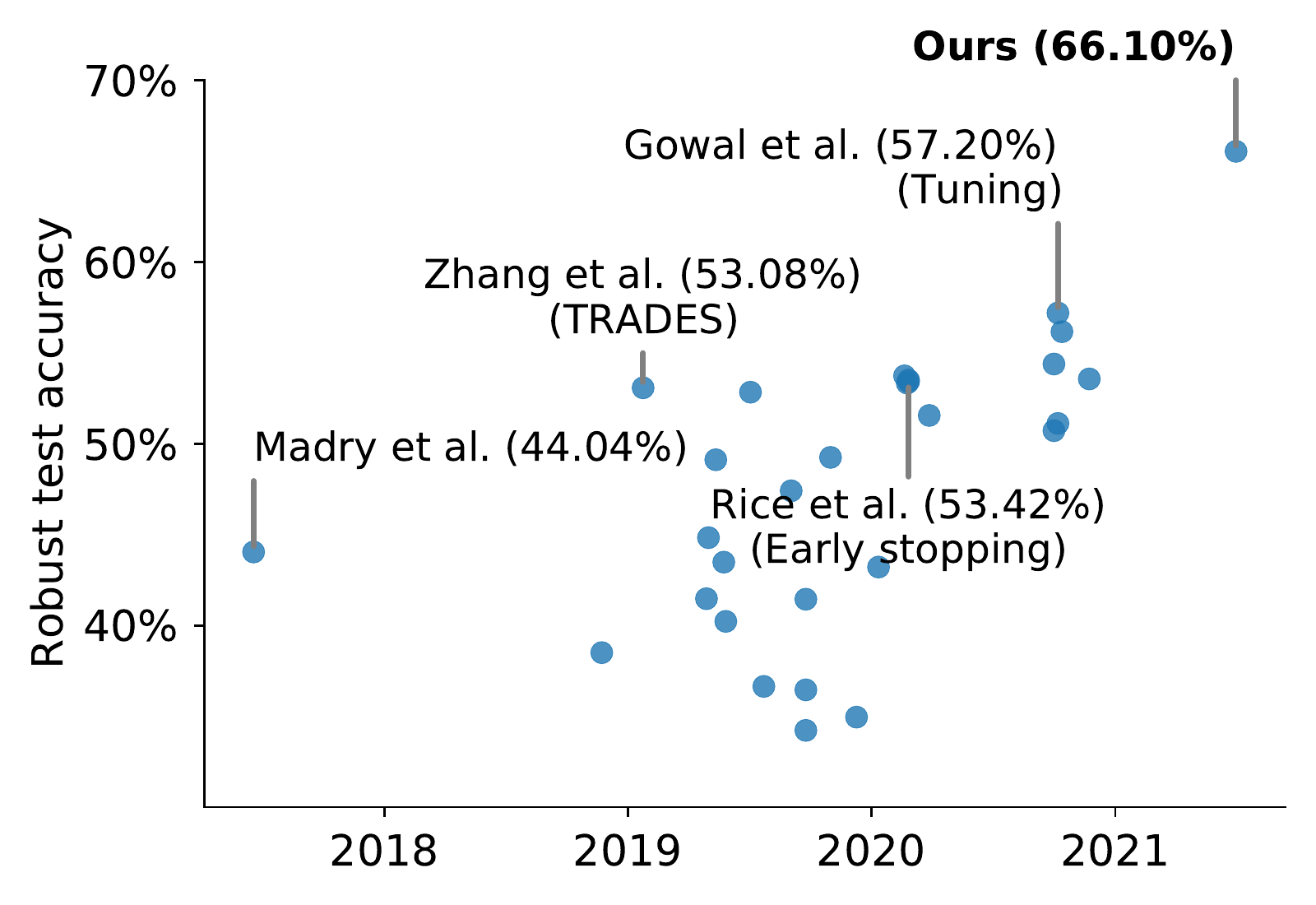}
  \captionof{figure}{Robust accuracy of models against \autoattack~\citep{croce_reliable_2020} on \cifar with \linf perturbations of size $8/255$ displayed in publication order.
Our method explores how generated data can be used to improve robust accuracy by +8.96\% without using any additional external data.
This constitutes the largest jump in robust accuracy in this setting. Our best model reaches a robust accuracy of 66.10\% against \textsc{AA+MT}~\citep{gowal_uncovering_2020}. \label{fig:history}}
\end{minipage}\hfill
\begin{minipage}{.45\textwidth}
  \centering
  \includegraphics[width=.65\textwidth]{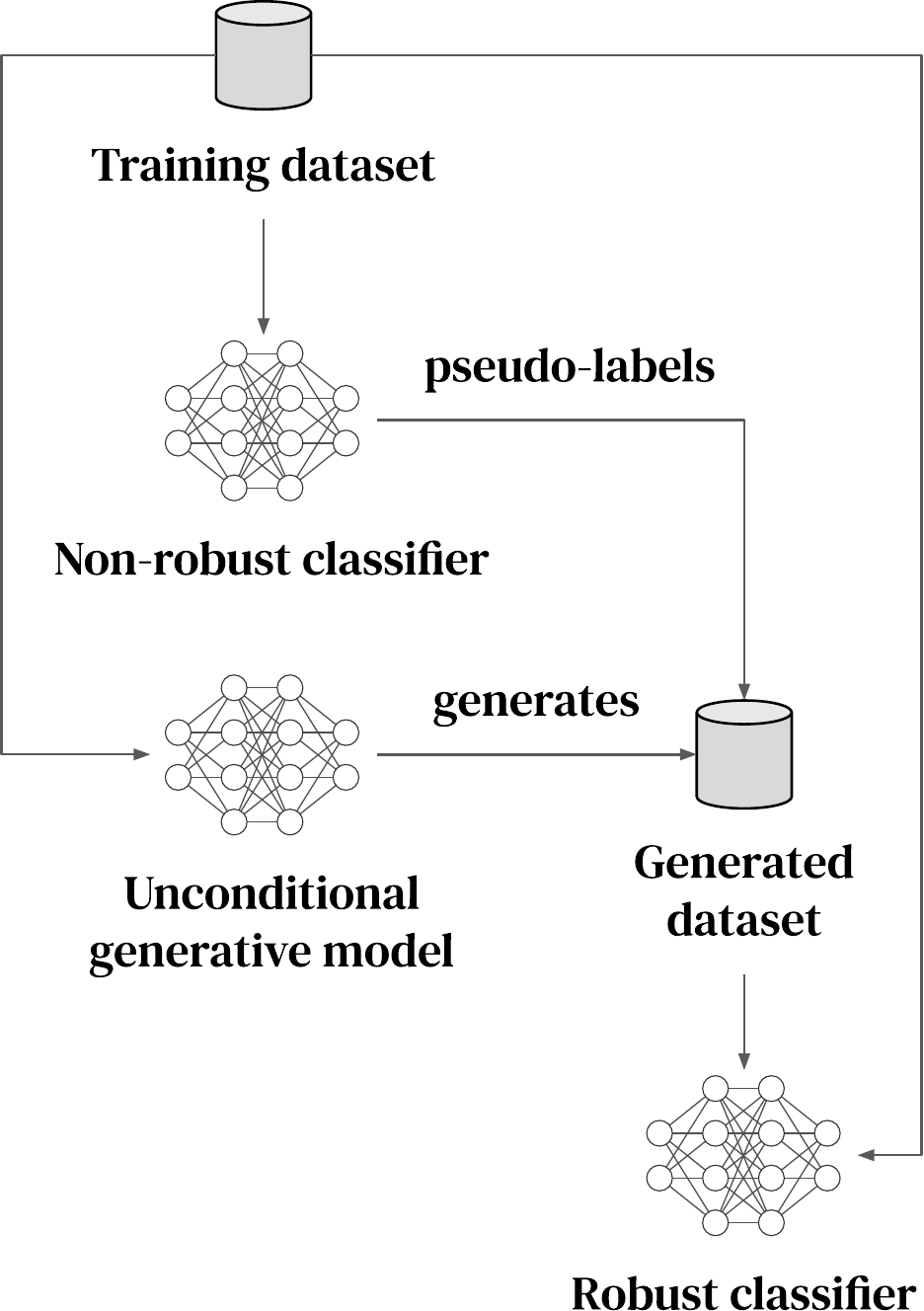}
  \captionof{figure}{Overview of our approach. Our method initially trains a generative model and a non-robust classifier. The non-robust classifier is used to provide pseudo-labels to the generated data. Finally, generated and original training data are combined to train a robust classifier. \label{fig:method}}
\end{minipage}
\end{figure}

The adversarial training procedure proposed by \citet{madry_towards_2017} feeds adversarially perturbed examples back into the training data.
It is widely regarded as one of the most successful method to train robust deep neural networks~\citep{gowal_uncovering_2020}, and it has been augmented in different ways -- with changes in the attack procedure~\citep{dong_boosting_2017}, loss function~\citep{mosbach_logit_2018,zhang_theoretically_2019} or model architecture~\citep{xie_feature_2018,zoran_towards_2019}.
We highlight the works by \citet{carmon_unlabeled_2019,uesato_are_2019,najafi_robustness_2019,zhai_adversarially_2019} who simultaneously proposed the use of additional unlabeled external data.
While the addition of external data helped boost robust accuracy by a large margin, progress in the setting without additional data has slowed (see \autoref{fig:history}).
On \cifar~\cite{krizhevsky_cifar-10_2014} against \linf perturbations of size $\epsilon = 8/255$, the best known model obtains a robust accuracy of 65.87\% when using additional data. The same model obtains a robust accuracy of 57.14\% without this data \citep{gowal_uncovering_2020}.
As a result, \uline{we ask ourselves whether it is possible to leverage the information contained in the original training set to a greater extent}.
This manuscript challenges the status-quo.
To the contrary of standard training where it is widely believed that generative models lack diversity and that the samples they produce cannot be used to train better classifiers \citep{ravuri_classification_2019}, we demonstrate both theoretically and experimentally that these generated samples can be used to improve robustness (using the approach described in \autoref{fig:method} and \autoref{sec:method}).
We make the following contributions:
\squishlist
\item We demonstrate in \autoref{sec:motivation} that it is possible to use low-quality random inputs (sampled from a conditional Gaussian fit of the training data) to improve robust accuracy on \cifar against \linf perturbations of size $\epsilon = 8/255$ (+0.93\% on a \textsc{Wrn}-28-10) and provide a justification and sufficient conditions in \autoref{sec:random_is_enough}.
\item We leverage higher quality generated inputs (i.e., inputs generated by \uline{generative models solely trained on the original data}), and study four recent generative models: the \gls{ddpm} \citep{ho2020denoising}, StyleGAN2~\citep{karras2020analyzing}, BigGAN~\citep{brock2018large} and the \gls{vdvae}~\citep{child2021vdvae} (\autoref{sec:generative_models}). We show that DDPM samples cover most closely the real data distribution (as measured by the distance to the test set in the Inception feature space).
\item Using images generated by the \gls{ddpm} allows us to reach a robust accuracy of 66.10\% on \cifar against \linf perturbations of size $\epsilon = 8/255$ (an improvement of +8.96\% upon the state-of-the-art). Notably, our best \cifar models beat all techniques that use additional data (see \autoref{sec:results}) and constitutes one of the largest improvements ever made in the setting without additional data. As a consequence, we demonstrate that it is possible to avoid the use of \tinyimages~\citep{80m} which has been withdrawn due to presence of offensive images.\footnote{\url{https://groups.csail.mit.edu/vision/TinyImages/}}
\squishend

\section{Related work}

\paragraph{Adversarial \lp norm-bounded attacks.}

Since \citet{szegedy_intriguing_2013,biggio2013evasion} observed that neural networks which achieve high accuracy are highly vulnerable to adversarial examples, the art of crafting increasingly sophisticated adversarial examples has received a lot of attention.
\citet{goodfellow_explaining_2014} proposed the \gls{fgsm} which generates adversarial examples with a single normalized gradient step.
It was followed by R+\gls{fgsm} \citep{tramer_ensemble_2017}, which adds a randomization step, and the \gls{bim} \citep{kurakin_adversarial_2016}, which takes multiple smaller gradient steps.

\paragraph{Adversarial training as a defense.}

Adversarial training \citep{madry_towards_2017} is widely regarded as one of the most successful methods to train deep neural networks robust to such attacks.
It has received significant attention and various modifications have emerged \citep{dong_boosting_2017,mosbach_logit_2018,xie_feature_2018}.
A notable work is TRADES \citep{zhang_theoretically_2019}, which balances the trade-off between standard and robust accuracy, and achieved state-of-the-art performance against \linf norm-bounded perturbations on \cifar.
More recently, the work from \citet{rice_overfitting_2020} studied \emph{robust overfitting} and demonstrated that improvements similar to TRADES could be obtained more easily using classical adversarial training with early stopping.
Finally, \citet{gowal_uncovering_2020} highlighted how different hyper-parameters (such as network size and model weight averaging) affect robustness.

\paragraph{Data-driven augmentations.}

Works, such as \emph{AutoAugment} \citep{cubuk_autoaugment:_2018} and related \emph{RandAugment}~\citep{cubuk2019randaugment}, learn augmentation policies directly from data.
These methods are tuned to improve standard classification accuracy and have been shown to work well on multiple datasets.
\emph{DeepAugment}~\citep{hendrycks_many_2020} explores how perturbations of the parameters of pre-trained image-to-image models can be used to generate augmented datasets that provide increased robustness to common corruptions \citep{hendrycks2018benchmarking}.
Similarly, generative models can be used to create novel views of images~\citep{plumerault2019controlling,jahanian2019steerability,harkonen_ganspace_2020} by manipulating them in latent space.
When optimized and used during training, these novel views reduce the impact of spurious correlations and improve accuracy~\citep{gowal_achieving_2019,wong2020learning}.
Most recently, \citet{laidlaw2020perceptual} proposed an adversarial training method based on bounding a neural perceptual distance (i.e., an approximation of the true perceptual distance).
While these works make significant contributions towards improving generalization and robustness to semantic perturbations, they do not improve robustness to \lp norm-bounded perturbations.

\paragraph{Robustness to \lp norm-bounded perturbations using generative modeling.}

Finally, we highlight works, such as \emph{Defense-GAN} \citep{samangouei2018defensegan} or \emph{ME-Net} \citep{yang2019menet}, which leverage data modeling techniques to create stronger defenses against \lp norm-bounded attacks.
Unfortunately, these techniques are not as robust as they seem and are broken by adaptive attacks~\citep{athalye_obfuscated_2018,tramer_adaptive_2020,croce_reliable_2020}.
Overall, to the best of our knowledge, there is little \citep{madaan2020learning} to no evidence that data augmentations or generative models can be used to improve robustness to \lp norm-bounded attacks.
In fact, generative models mostly lack diversity and it is widely believed that the samples they produce cannot be used to train classifiers to the same accuracy than those trained on original datasets \citep{ravuri_classification_2019}.
We differentiate ourselves from earlier works by leveraging additional generated samples for training rather than modifying the defense procedure, and by establishing sufficient conditions under which such samples improve robustness.

\section{Adversarial training using generated data}

The rest of this manuscript is organized as follows.
In this section, we provide an overview of adversarial training, demonstrate using a motivational example that low-quality generated data can be leveraged to improve robustness to adversarial examples, and describe our method.
In \autoref{sec:random_is_enough}, we detail sufficient conditions that explain why generated samples can improve robustness and explore the limitations of our approach.
In \autoref{sec:generative_models}, we analyze four complementary and recent generative models in the context of our method.
Finally, we provide experimental results in \autoref{sec:experiments}.

\subsection{Adversarial training}

For classification tasks, \citet{madry_towards_2017} propose to find model parameters $\vtheta$ that minimize the adversarial risk:
\begin{equation}
\argmin_\vtheta \E_{(\vx,y) \sim \realdata} \left( \maximize_{\vdelta \in \gS} \left [ f(\vx + \vdelta; \vtheta) \neq y \right] \right)
\label{eq:adversarial_risk}
\end{equation}
where $\realdata$ is a data distribution over pairs of examples $\vx$ and corresponding labels $y$, $f(\cdot; \vtheta)$ is a model parametrized by $\vtheta$, $[\cdot]$ is the Iverson bracket notation and corresponds to the $0-1$ loss, and $\gS$ defines the set of allowed perturbations.
For \lp norm-bounded perturbations of size $\epsilon$, the perturbation set is defined as $\gS_p = \{ \vdelta ~|~ \| \vdelta \|_p \leq \epsilon \}$.
Hence, for \linf norm-bounded perturbations $\gS = \gS_\infty$ and for \ltwo norm-bounded perturbations $\gS = \gS_2$.
In the rest of this manuscript, we use $\epsilon_p$ to denote $\ell_p$ norm-bounded perturbations of size $\epsilon$ (e.g., $\epsilon_\infty = 8/255$).

In practice, given a training set $\train$, the adversarial training procedure replaces the $0-1$ loss with the cross-entropy loss $\xent$ and is formulated as
\begin{equation}
\argmin_\vtheta \E_{(\vx,y) \in \train} \left( \maximize_{\vdelta \in \gS} \xent(f(\vx + \vdelta; \vtheta), y) \right).
\label{eq:at}
\end{equation}

\subsection{Generated data can improve robust generalization}
\label{sec:motivation}

Data augmentations can reduce the generalization error of standard (non-robust) training~\citep{devries2017improved,zhang2017mixup,cubuk_autoaugment:_2018,cubuk2019randaugment}.
However, to the contrary of standard training, augmentations beyond random flips and crops \citep{he2015deep} -- such as \emph{Cutout} \citep{devries2017improved}, \emph{mixup} \citep{zhang2017mixup}, \emph{AutoAugment} \citep{cubuk_autoaugment:_2018} or \emph{RandAugment} \citep{cubuk2019randaugment} -- have been unsuccessful in the context of adversarial training \citep{rice_overfitting_2020,gowal_uncovering_2020,wu2020adversarial}.
The gap in robust accuracy between models trained with and without additional data suggests that common augmentation techniques, which tend to produce augmented views that are close to the original image they augment, are intrinsically limited in their ability to improve robust generalization.
In other words, augmented samples are diverse (if the training set is diverse), but not complementary to the training set.
This phenomenon is particularly exacerbated when training adversarially robust models which are known to require an amount of data polynomial in the number of input dimensions~\citep{schmidt_adversarially_2018}.

\begin{figure*}[t]
\centering
\subfigure[\label{fig:motivation}]{\includegraphics[width=0.45\textwidth]{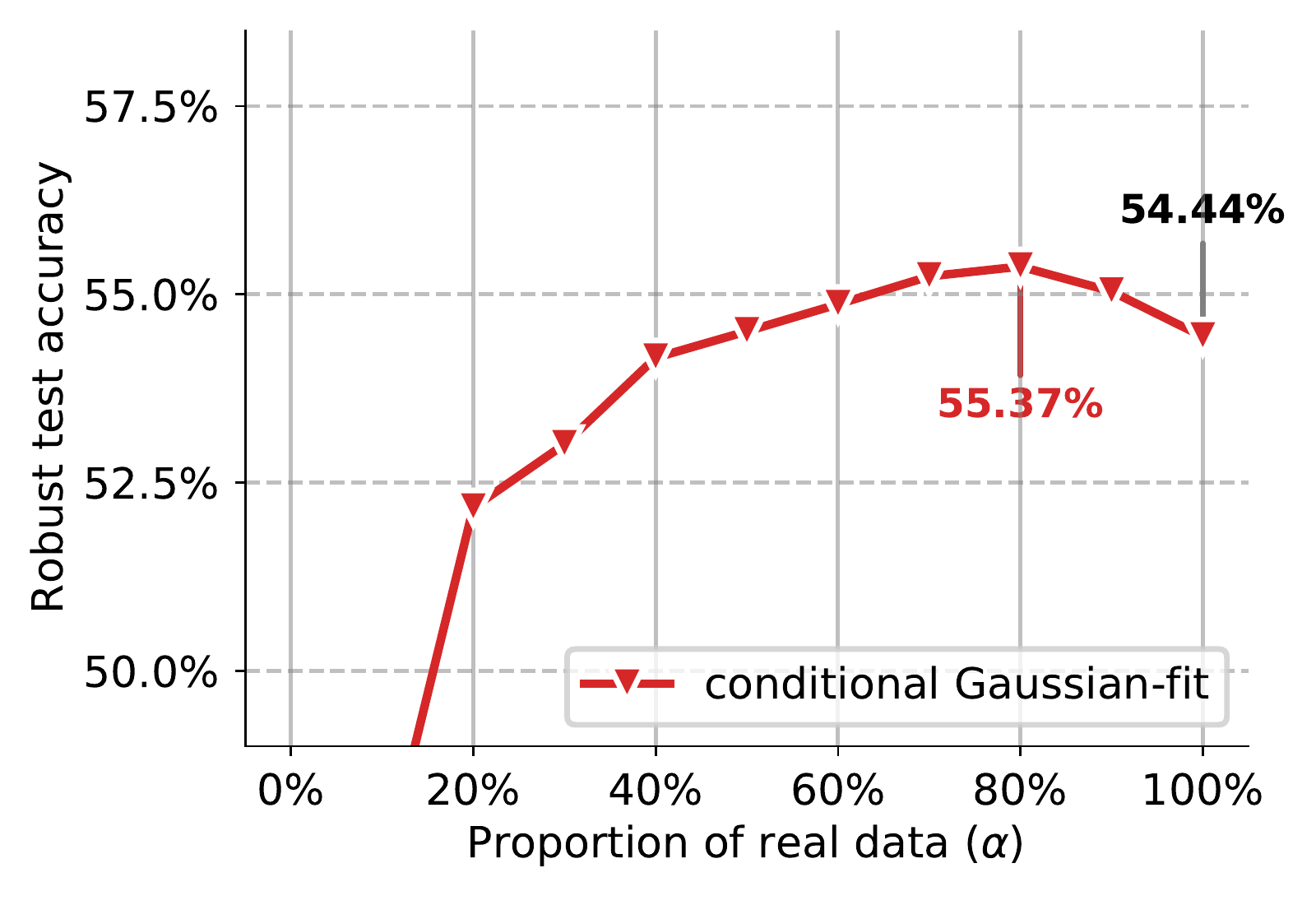}} \hspace{1cm}
\subfigure[\label{fig:gaussian_samples}]{\includegraphics[width=0.3\textwidth]{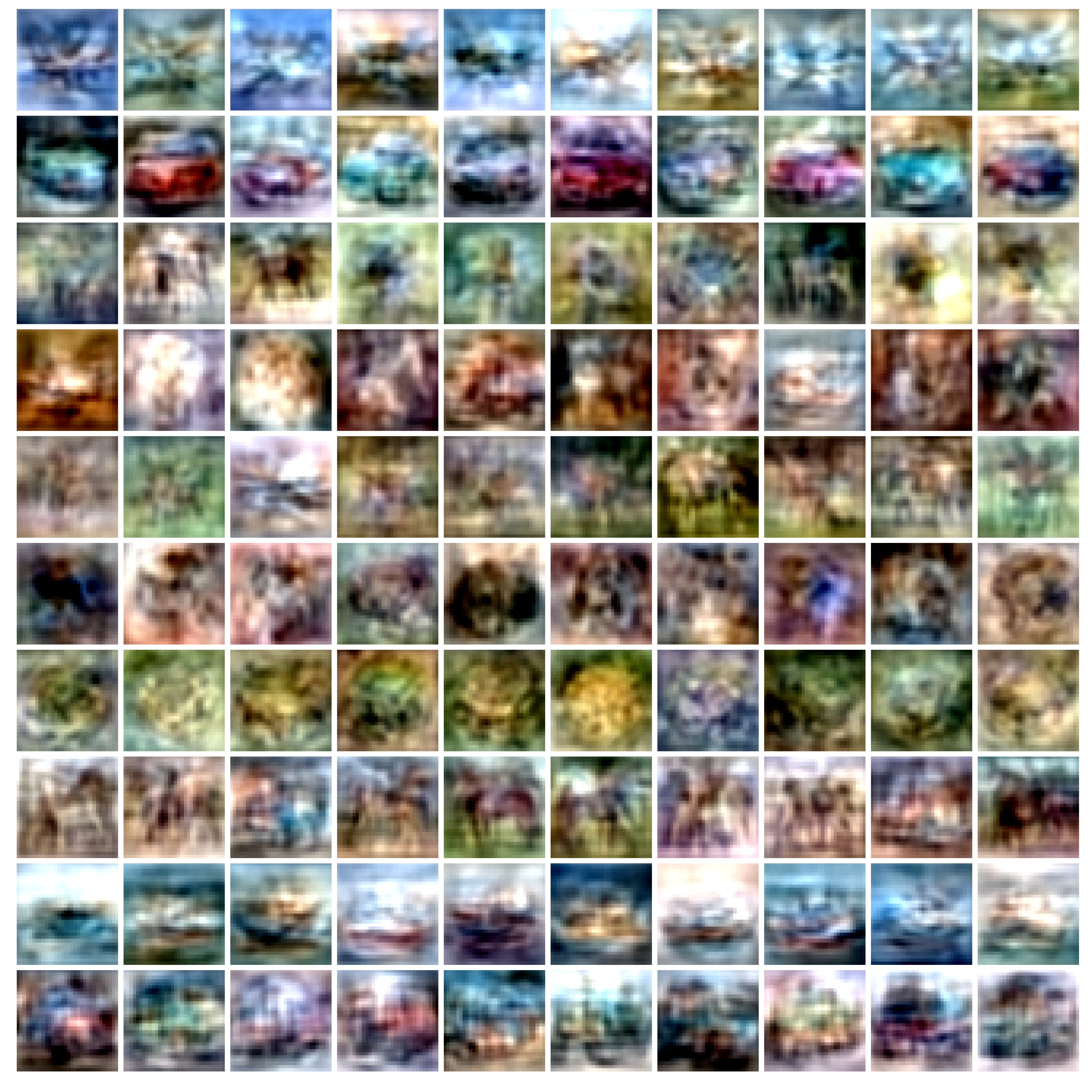}}
\caption{Low-quality random inputs can improve robustness. Panel~\subref{fig:motivation} shows the robust test accuracy (against \textsc{AA+MT}~\citep{gowal_uncovering_2020}) of a \textsc{Wrn}-28-10 against $\epsilon_\infty = 8/255$ on \cifar trained with additional data randomly sampled from a class-conditional Gaussian fit of the training data. We compare how the proportion of original \cifar and generated images affects robustness (0\% means generated samples only, while 100\% means original \cifar train set only). Panel~\subref{fig:gaussian_samples} shows some of the class-conditional Gaussian samples that are used during training. \label{fig:motivation_all}}
\end{figure*}

We hypothesize that, \uline{to improve robust generalization, it is critical to use additional training samples (augmented or generated) that are diverse and that complement the original training set} (in the sense that these new samples should ideally come from the same underlying distribution as the training set but should not duplicate the training set).
To test this hypothesis, we propose to use samples generated from a simple class-conditional Gaussian fit of the training data.
By construction, such samples (shown in \autoref{fig:gaussian_samples}) are extremely blurry but diverse.
We proceed by fitting a multivariate Gaussian to each set of 5K training images corresponding to each class in \cifar.
For each class, we sample 100K images resulting in a new dataset of 1M datapoints (no further filtering is applied).
In \autoref{fig:motivation}, we show the performance of various robust models trained by decreasing the proportion of real samples present in each batch from 100\% (original data only) to 0\% (generated data only).
Decreasing this proportion reduces the importance of the original data.
We observe that all proportions between 50\% and 90\% provide improvements in robust accuracy.
Most surprisingly, the optimal proportion of 80\% provides an absolute improvement of +0.93\%, which is an improvement comparable in size to the ones provided by model weight averaging or TRADES \citep{gowal_uncovering_2020}.
As we show in \autoref{sec:random_is_enough}, the drop in robust accuracy for proportions below 50\% is expected in the capacity-limited regime.
This experiment directly motivates our method.

\subsection{Method}\label{sec:method}

Given access to a pre-trained \uline{non-robust} classifier $\fnr$ and an \uline{unconditional generative} model approximating the true data distribution $\realdata$ by a distribution $\smash{\generateddata}$, we would like to train a \uline{robust} classifier $f(\cdot; \vtheta)$ parameterized by $\vtheta$.
We propose the following optimization problem:
\begin{equation}
    \argmin_\vtheta \alpha \cdot \!\!\!\!\mathop{\E}_{(\vx,y) \in \train} \left( \maximize_{\vdelta \in \gS} \xent(f(\vx + \vdelta; \vtheta), y) \right) + (1 - \alpha) \cdot \mathop{\E}_{\vx \sim \generateddata} \left( \maximize_{\vdelta \in \gS} \xent(f(\vx + \vdelta; \vtheta), \fnr(\vx)) \right) \label{eq:method}
\end{equation}
where $\alpha$ corresponds to a mixing factor that blends examples from the training set with those that are generated.
When $\alpha$ is set to one, our method reverts back to the original adversarial training formulation in \autoref{eq:at}.
When $\alpha$ is set to zero, our method only uses generated samples with their corresponding pseudo-labels.
In practice, for efficiency, rather than generating samples on-the-fly, we pre-generate samples offline.
Hence, both the original training set $\train$ and generated set $\smash{\generateddata}$ contain a finite number of samples.
We have the advantage, however, to be able to generate significantly more samples than present in the original training set.
In \autoref{sec:additional_results}, we evaluate how varying the number of generated samples impacts adversarial robustness.

Overall, the complete method is described in \autoref{fig:method} and is composed of three steps: \textit{(i)} it starts by training the non-robust classifier and generative model on the original training set (for \cifar, that corresponds to 50K images only); \textit{(ii)} then, the generated dataset is constructed by drawing samples from the generative model and pseudo-labeling them using the non-robust classifier; \textit{(iii)} finally, the robust classifier is trained using both the original training set and the generated dataset using \autoref{eq:method}.

\section{Randomness might be enough}
\label{sec:random_is_enough}

In this section, we formalize our notation (\autoref{sec:theory_setup}), and provide three sufficient conditions that explain why generated data can improve robustness (\autoref{sec:conditions}).
In summary, \textit{(i)} the pre-trained, non-robust classifier $\fnr$ used for pseudo-labeling must be accurate, \textit{(ii)} the likelihood of sampling examples that are adversarial to this non-robust classifier must be low, and \textit{(iii)} the generative model must be able to sample images from the true data distribution with non-zero probability.

\subsection{Setup}\label{sec:theory_setup}

Given a ground-truth function $\fs$, we would like to find optimal parameters $\vtheta^\star$ for $f(\cdot; \vtheta^\star)$ that minimize the adversarial risk,
\begin{equation}
    \vtheta^\star = \argmin_\vtheta \E_{\vx \sim \realdata} \left( \maximize_{\vdelta \in \gS} \left [ f(\vx + \vdelta; \vtheta) \neq \fs(\vx) \right] \right),
    \label{eq:risk}
\end{equation}
without access to the true data distribution $\smash{\realdata}$ or the ground-truth classifier $\fs$.
As such, we replace the distribution $\smash{\realdata}$ with an approximated distribution $\smash[b]{\generateddata}$ (from a generative model) and use a pre-trained non-robust classifier $\fnr$ instead of $\fs$ (see \autoref{sec:method}).
This results in sub-optimal parameters
\begin{equation}
    \hat{\vtheta}^\star = \argmin_\vtheta \E_{\vx \sim \generateddata} \left( \maximize_{\vdelta \in \gS} \left[ f(\vx + \vdelta; \vtheta) \neq \fnr(\vx) \right] \right).
    \label{eq:approx_risk}
\end{equation}
We introduce the unknown probability measure $\mu$ corresponding to the true data distribution $\smash{\realdata}$ and defined over the set of inputs $\gA \subseteq \R^n$ (where $n$ is the input dimensionality), as well as the known probability measure $\hat{\mu}$ corresponding to the approximated distribution $\smash{\generateddata}$.
The set of relevant inputs $\gX \subseteq \gA$ (i.e., the set of \emph{realistic} images for which we would like to enforce robustness) is the support of $\mu$ such that $\mu(\gX) = 1$ and $\forall \gW \subseteq \gX, \mu(\gW) > 0$ if $\gW$ is non-empty.
We assume that each input $\vx \in \gX$ can be assigned a label $y = \fs(\vx)$ where $\fs: \gX \mapsto \gY$ is the ground-truth classifier (only valid for \emph{realistic} images) and $\gY \in \smash{2^{\mathbb{Z}}}$ is the set of labels.
Finally, given a perturbation set $\gS$, we restrict labels such that there exists no \emph{realistic} image within the perturbation set of another that has a different label; i.e., for $\vx \in \gX$, for all $\vdelta \in \{ \vdelta' \in \gS | \vx + \vdelta' \in \gX \}$ we have $\fs(\vx) = \fs(\vx + \vdelta)$.

\subsection{Limitations and sufficient conditions}
\label{sec:conditions}

To understand the limitations of our approach, it is useful to think about idealized sufficient conditions that would allow the sub-optimal parameters $\smash{\hat{\vtheta}^\star}$ to approach the performance of the optimal parameters $\smash{\vtheta^\star}$.
First, we concentrate on the capacity-limited regime and later extrapolate to the infinite-capacity, infinite-compute regime to gain more insights.
The first sufficient condition concerns the pre-trained non-robust classifier and holds for both regimes.

\begin{condition}[accurate non-robust classifier]\label{cond:accuracy}
The pre-trained non-robust classifier $\fnr: \gA \mapsto \gY$ must be accurate on all \emph{realistic} inputs $\vx \in \gX$: $\forall \vx \in \gX, \fnr(x) = \fs(x)$.
\end{condition}

Indeed, if we had access to the true distribution $\smash{\realdata}$, \autoref{eq:risk} and \autoref{eq:approx_risk} could be made equal by setting $\fs(x) = \fnr(x)$.
\footnote{
In \autoref{sec:experiments}, we show that sub-optimal parameters $\hat{\vtheta}^\star$ that improve upon those obtained by \autoref{eq:at} can be obtained even when the non-robust classifier $\fnr$ is not perfect.
In our experiments, we use a classifier that achieves 95.68\% on the \cifar test set.
}
In the capacity-limited regime, when \autoref{cond:accuracy} is satisfied, the problem reduces to a robust generalization problem.
This problem is widely studied~\cite{cullina_pac_2018,tu2019theoretical,bhagoji2019lower} and one can show that the adversarial risk is bounded by the Wasserstein distance between the training distribution $\smash{\generateddata}$ and true data distribution $\smash{\realdata}$ (under mild assumptions)~\cite{lee2018minimax}.
In other words, as $\smash{\generateddata}$ approaches $\smash{\realdata}$, we expect the robust accuracy of $f(\cdot; \smash{\hat{\vtheta}^\star})$ to approach the one of $f(\cdot; \smash{\vtheta^\star})$.
This intuitively leads to the second sufficient condition and \autoref{prop:limited}.

\begin{condition}[accurate approximated distribution]\label{cond:gap}
The approximated data distribution $\generateddata$ and true data distribution $\realdata$ must be equivalent: $\mu(\gW) = \hat{\mu}(\gW)$ for all measurable subset $\gW \subseteq \gX$.
\end{condition}

\begin{proposition}[capacity-limited regime]\label{prop:limited}
\autoref{cond:accuracy} and \autoref{cond:gap} are sufficient conditions that allow the sub-optimal parameters $\smash{\hat{\vtheta}^\star}$ to match the performance of the optimal parameters $\smash{\vtheta^\star}$.
\end{proposition}

Together, \autoref{cond:accuracy} and \ref{cond:gap} provide sufficient conditions for the capacity-limited regime (see proof in \autoref{sec:proof}).
\autoref{cond:gap} (and associated bounds from~\cite{lee2018minimax}) generally indicates that the robust accuracy of the classifier $f(\cdot; \smash{\hat{\vtheta}^\star})$ should increase as the quality of the generative model that provides the approximated distribution $\smash{\generateddata}$ improves.
However, these two conditions do not provide a satisfying answer when it comes to understanding why seemingly random data can help improve robustness (as demonstrated in \autoref{sec:motivation}).
To help our understanding, it is worth analyzing the consequence of increasing the capacity of $f$.
In particular, in the infinite-capacity regime, \autoref{cond:gap} can be relaxed and replaced by the following two conditions, and \autoref{prop:limited} becomes \autoref{prop:infinite}.

\begin{condition}[unlikely adversarial examples]\label{cond:unlikely}
It is not possible to sample a point $\vx \sim \generateddata$ outside the \emph{realistic} set $\gX$ such that it is adversarial to $\fnr$:
$\hat{\mu}(\gW) = 0$ on the measurable subset $\gW = \{ x + \delta ~|~ x \in \gX, \delta \in \gS, f_\textrm{NR}(x+\delta) \neq f_\textrm{NR}(x) \}$.
\end{condition}

\begin{condition}[sufficient coverage]\label{cond:coverage}
The likelihood of any finite sample in the set of \emph{realistic} inputs $\mathcal{X}$ obtained from $\smash{\generateddata}$ should be non-zero under the measure $\hat{\mu}$: $\hat{\mu}(\gW) > 0$ for all open measurable subsets $\gW \subseteq \gX$.
\end{condition}

\begin{proposition}[infinite-capacity regime]\label{prop:infinite}
\autoref{cond:accuracy}, \autoref{cond:unlikely} and \autoref{cond:coverage} are sufficient conditions that allow the sub-optimal parameters $\smash{\hat{\vtheta}^\star}$ to match the performance of the optimal parameters $\smash{\vtheta^\star}$ when the model $f$ has infinite capacity.
\end{proposition}

\autoref{cond:unlikely} enforces that labels are non-conflicting within the perturbation set of a \emph{realistic} input,\footnote{
Unless the generative model is trained to produce adversarial examples, random sampling is unlikely to produce images that are adversarial \cite{brunner2019guessing}.
In fact, even strong black-box adversarial attacks require thousands of model queries to find adversarial examples.}
while \autoref{cond:coverage} guarantees that \emph{realistic} inputs appear with enough frequency during training.
Together, \autoref{cond:accuracy}, \ref{cond:unlikely} and \ref{cond:coverage} do not only provide sufficient conditions for the infinite-capacity regime (see proof in \autoref{sec:proof}),
but also explain why samples generated by a simple class-conditional Gaussian-fit can be used to improve robustness.
Indeed, they imply that it is not necessary to have access to either the true data distribution or a perfect generative model when given enough compute and capacity.
However, when compute and capacity are limited, it is critical that the optimization in \autoref{eq:approx_risk} focuses on \emph{realistic} inputs and that the distribution $\smash{\generateddata}$ be as close as possible to the true distribution $\realdata$.
In practice, this translates to the fact that better generative models (such as \gls{ddpm}) can be used to achieve better robustness.
We have relegated a discussion about the theoretical impact of the mixing factor $\alpha$ in \autoref{sec:impact_alpha}. Briefly stated, increasing $\alpha$ improves the realism of training samples (since the training samples mostly come from the original training set), but comes at the cost of a reduction in complementarity with the training set.

\section{Generative models}
\label{sec:generative_models}

The derivations from \autoref{sec:random_is_enough} and the experiment performed in \autoref{sec:motivation} strongly suggest that generative models, which are capable of creating novel images~\citep{dalle}, are viable augmentation candidates for adversarial training.

\paragraph{Generative models considered in this work.}

In this work, we limit ourselves to generative models that are solely trained on the original train set, as we focus on how to improve robustness in the setting without external data.
We consider four recent and fundamentally different models: \textit{(i)} BigGAN~\citep{brock2018large}: one of the first large-scale application of \glspl{gan} which produced significant improvements in \gls{fid} and \gls{is} on \cifar (as well as on \imagenet); \textit{(ii)} \gls{vdvae}~\citep{child2021vdvae}: a hierarchical \gls{vae} which outperforms alternative \gls{vae} baselines; \textit{(iii)} StyleGAN2~\citep{karras2020analyzing}: an improved version of StyleGAN which borrows interesting properties from the style transfer literature; and \textit{(iv)} \gls{ddpm}~\citep{ho2020denoising}: a diffusion probabilistic model based on Langevin dynamics that reaches state-of-the-art \gls{fid} on \cifar.\footnote{We use \gls{vdvae}, StyleGAN2 and \gls{ddpm} checkpoints available online and train our own BigGAN.}
As we have done for the simpler class-conditional Gaussian-fit, for each model, we sample 100K images per class, resulting in 1M images in total (see \autoref{sec:details_data_augment} for details).
Samples are shown in \autoref{sec:generated_samples}.

\begin{table}[t]
\caption{
Complementarity and coverage of augmented and generated samples.
We sample 10K images from the train set and various different generative models.
For each sample in each set, we find its closest neighbor in Inception feature space (obtained after the pooling layer).
To estimate \uline{complementarity}, we report the proportion of samples with a nearest neighbor in either the train set, test set or the sampled set itself.
To estimate \uline{coverage}, we report the proportion of unique neighbors in the train and test set.
We also include the IS and FID computed from 50K samples from each set and the robust accuracy obtained by a \textsc{Wrn}-28-10 models trained on 1M samples (\autoref{sec:experiments}).
\label{table:similarity_train_test_self_gen}}
\begin{center}
\resizebox{\textwidth}{!}{
\begin{tabular}{l|ccc|cc|cc|c}
    \hline
    \cellcolor{header} & \multicolumn{3}{c|}{\cellcolor{header} \textsc{Complementarity}} & \multicolumn{2}{c|}{\cellcolor{header} \textsc{Coverage}}  & \multicolumn{2}{c|}{\cellcolor{header} \textsc{Inception Metrics}} & \cellcolor{header}\textsc{Robust} \TBstrut \\
    \cellcolor{header} \textsc{Setup} & \cellcolor{header} \textsc{Train} & \cellcolor{header} \textsc{Test} & \cellcolor{header} \textsc{Self} & \cellcolor{header} \textsc{Train} & \cellcolor{header} \textsc{Test} & \cellcolor{header} \textsc{Is} $\uparrow$ & \cellcolor{header} \textsc{Fid} $\downarrow$ & \cellcolor{header}\textsc{Accuracy} $\uparrow$ \TBstrut \\
    \hline
    \emph{mixup}~\citep{zhang2017mixup} & 90.34\% & 3.91\% & 5.75\% & 90.43\% & 45.61\% & $9.33 \pm 0.22$ & 7.71 & \TBstrut \\
    \hline
    Class-conditional Gaussian-fit & 0.13\% & 0.22\% & 99.65\% & 12.36\% & 12.24\% & $3.64 \pm 0.03$ & 117.62 & 55.37\% \Tstrut \\
    VDVAE~\citep{child2021vdvae} & 11.97\% & 12.14\% & 75.89\% & 34.20\% & 33.76\% & $6.88 \pm 0.05$ & 26.44 & 55.51\% \\
    BigGAN~\citep{brock2018large} & 14.97\% & 14.81\% & 70.22\% & 38.86\% & 39.06\% & $9.73 \pm 0.07$ & 13.78 & 55.99\% \\
    StyleGAN2~\citep{karras2020analyzing} & 28.13\% & 27.22\% & 44.65\% & 50.16\% & 48.29\% & $10.04 \pm 0.11$ & 2.57 & 58.17\% \\
    DDPM~\citep{ho2020denoising} & 29.29\% & 29.17\% & 41.54\%  & 49.07\% & 49.10\% & $9.50 \pm 0.14$ & 3.15 & 60.73\% \Bstrut \\
    \hline
\end{tabular}
}
\end{center}
\end{table}

\paragraph{Analysis of complementary and coverage.}

In \autoref{table:similarity_train_test_self_gen}, we evaluate how close \autoref{cond:gap} and \autoref{cond:coverage} are to be satisfied in practice.
To do so, we sample 10K images from each generative model.
We also sample 10K images from the \cifar training set, and apply \emph{mixup} to them as a point of comparison.\footnote{According to prior work, \emph{mixup} is unable to improve robust accuracy beyond the one obtained with random cropping/flipping when using early stopping~\citep{rice_overfitting_2020}. \autoref{table:similarity_train_test_self_gen_aug} in the appendix shows more data augmentations.}
We observe that \emph{mixup} achieves a similar \gls{is} to the BigGAN and \gls{ddpm} models.
In the left-most set of three columns, for each augmented or generated sample, we report whether its closest neighbor in the Inception\footnote{Using the LPIPS~\cite{zhang2018unreasonable} feature space (see \autoref{table:similarity_train_test_self_gen_lpips} in the appendix) provides similar conclusions.} feature space belongs to the train set, test set or the generated set itself (more details are available in \autoref{sec:details_data_augment}).
An ideal generative model should create samples that are equally likely to be close to images from each set.
We observe that \emph{mixup} tends to produce samples that are too close to the original train set and that lack complementarity, potentially explaining its limited usefulness in terms of improving adversarial robustness.
Meanwhile, generated samples (including those from the class-conditional Gaussian-fit) are much more likely to be close to images of the test set.
We also observe that the \gls{ddpm} neighbor distribution matches more closely the ideal uniform distribution.
Images generated by BigGAN and \gls{vdvae} tend to have their nearest neighbor among themselves which indicates that these samples are either far from the train and test distributions or produce overly similar samples.
Images generated by StyleGAN2, which reach an \gls{fid} of 2.57 and \gls{is} of 10.07 that are better than the \gls{ddpm} scores, have a slightly worse neighbor distribution (indicating a slight memorization of the training set).
The middle two columns measure the ratio of unique neighbors that are matched in the train and test set.
This provides a rough approximation of coverage.
We observe a similar trend where samples from the \gls{ddpm} seem to provide a better coverage of the true data distribution.
Note that, while these numbers rely on an inaccurate distance measure (i.e., Euclidean distance in Inception feature space) and should be taken with a grain of salt, they correlate well with the results obtained from our experiments.
For example, models trained with StyleGAN2 samples obtain a lower robust accuracy than those trained with \gls{ddpm} samples -- despite obtaining better \gls{fid} and \gls{is}.

\section{Experiments}
\label{sec:experiments}
\label{sec:results}

The experimental setup is explained in \autoref{sec:experimental_setup}.
We use Residual Networks (ResNets) and \glspl*{wrn}~\citep{he2015deep,zagoruyko2016wide} with Swish/SiLU~\citep{hendrycks2016gaussian} activations.
We use stochastic weight averaging~\citep{izmailov_averaging_2018} with a decay rate of $0.995$.
For adversarial training, we use TRADES~\citep{zhang_theoretically_2019} with 10 \gls*{pgd} steps.
We train for $400$ \cifar-equivalent epochs with a batch size of $1024$ (i.e., $19$K steps).
We evaluate our models against \autoattack~\citep{croce_reliable_2020} and \multitargeted~\citep{gowal_alternative_2019}, which is denoted \textsc{AA+MT}~\citep{gowal_uncovering_2020}.
For comparison, we trained ten \wrn-28-10 models on \cifar (without additional generated samples) against $\epsilon_\infty = 8/255$.
The resulting robust accuracy is 54.44$\pm$0.39\%, thus showing a relatively low variance.
Furthermore, as we will see, our best models are well clear of the threshold for statistical significance.
On \cifar against $\epsilon_\infty = 8/255$ without additional generated samples a ResNet-18 achieves a robust accuracy of 50.64\% and a \wrn-70-16 achieves 57.14\%.
Unless stated otherwise, all results pertain to \cifar.

\begin{figure*}[t]
\centering
\subfigure[Condition \ref{cond:accuracy} \label{fig:non_robust_accuracy_v2}]{\includegraphics[width=0.32\textwidth]{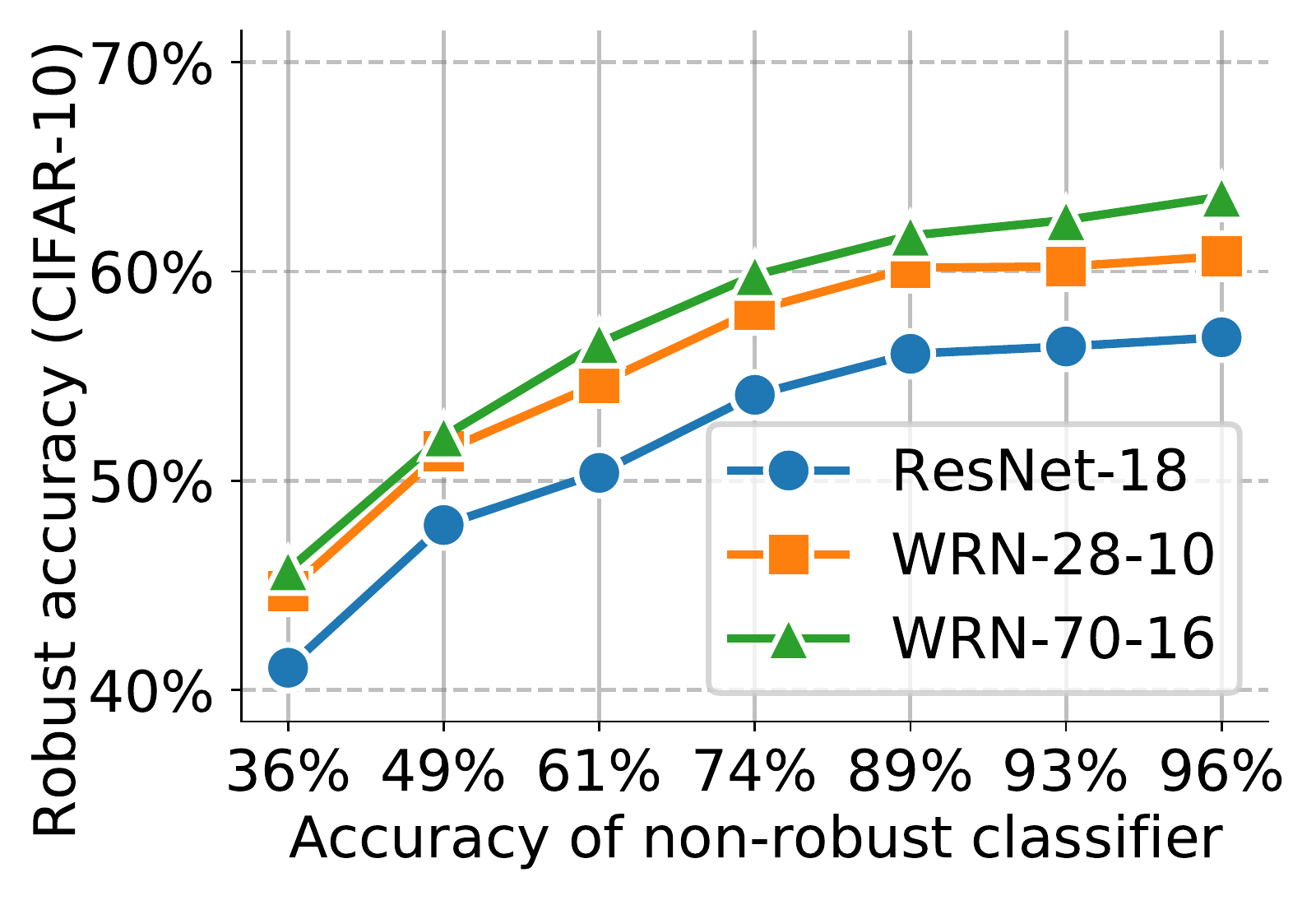}}
\subfigure[Condition \ref{cond:gap} \label{fig:biggan_generated_ratio}]{\includegraphics[width=0.32\textwidth]{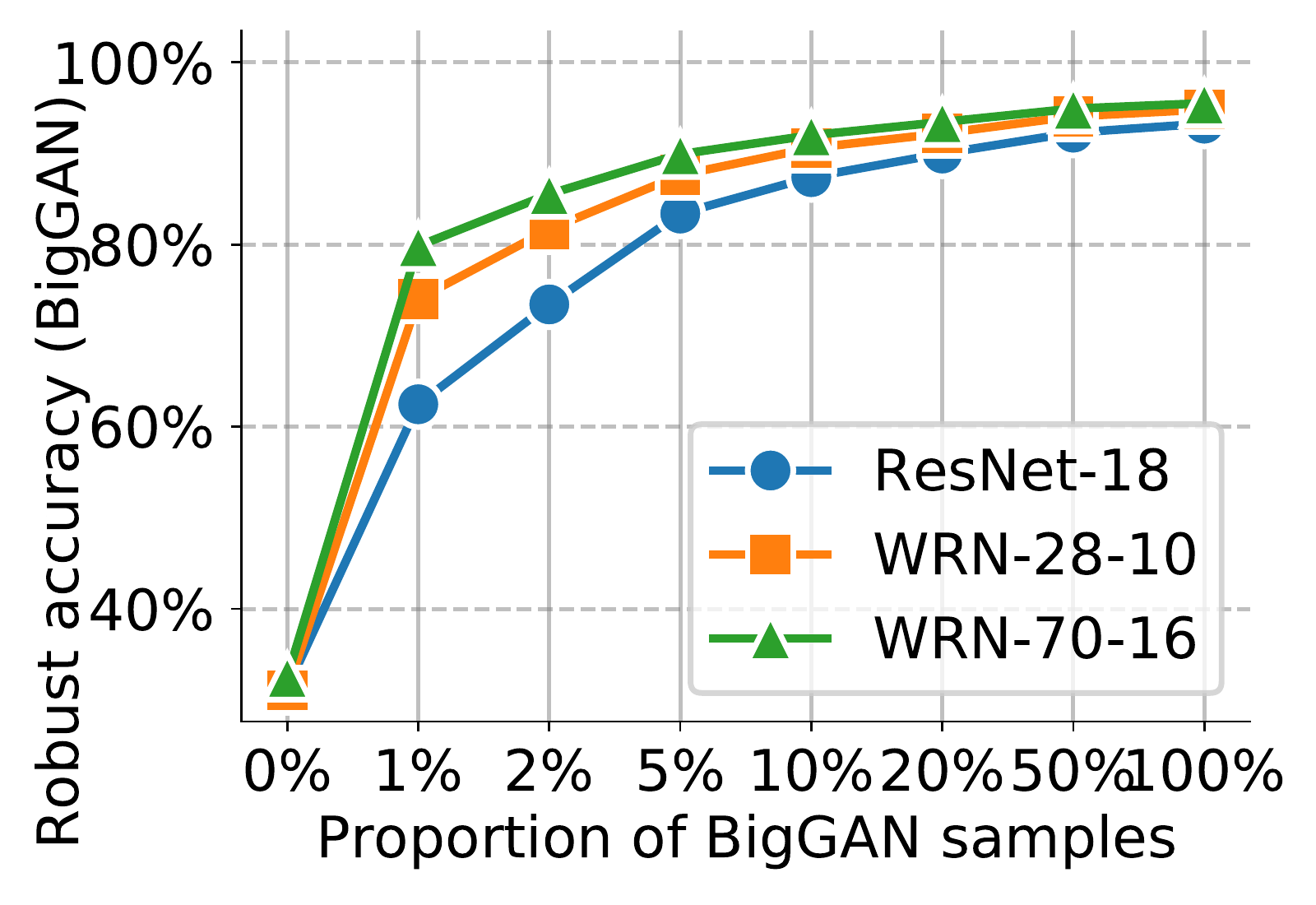}}
\subfigure[Condition \ref{cond:coverage} \label{fig:biggan_set_coverage}]{\includegraphics[width=0.32\textwidth]{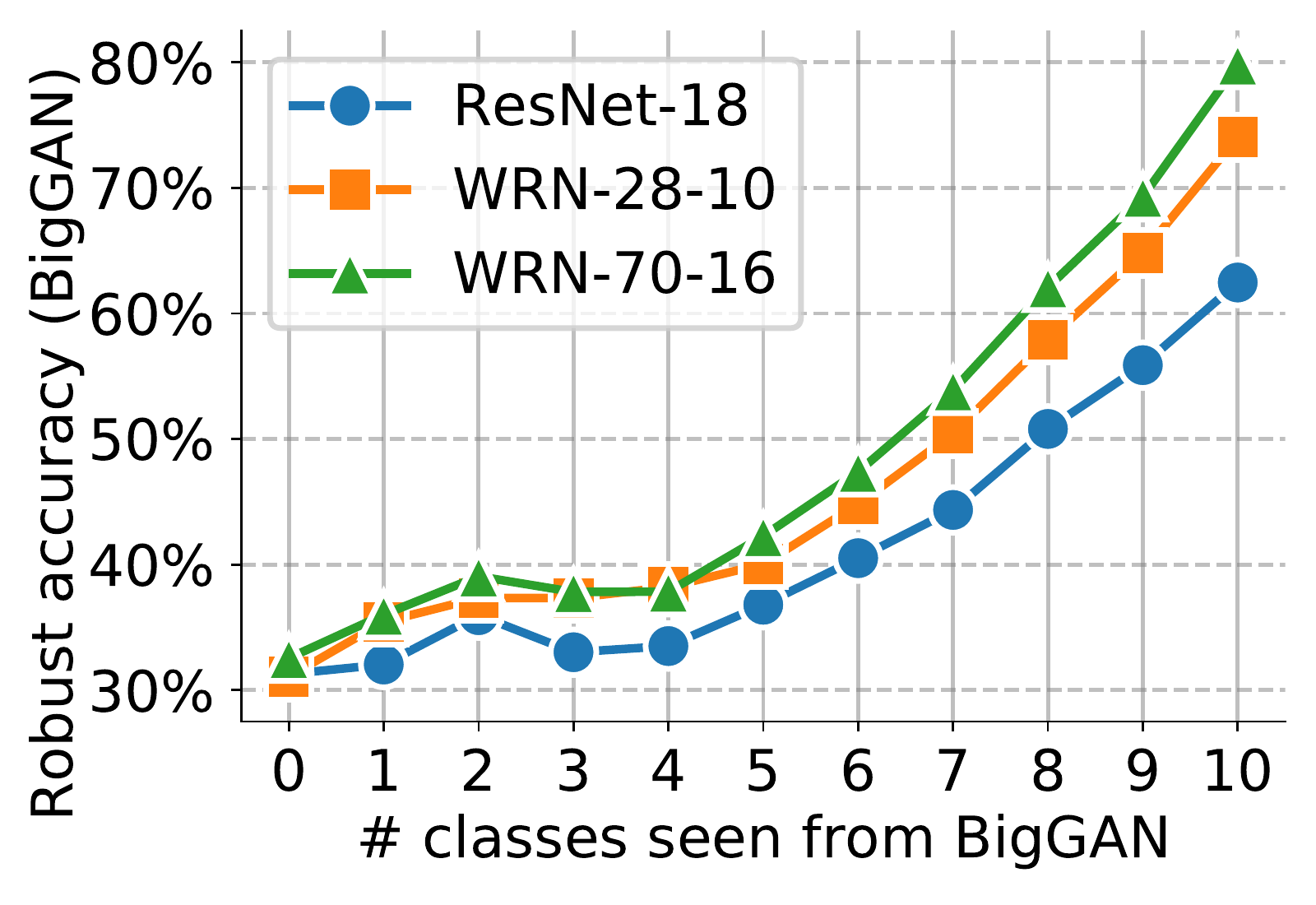}}
\caption{Impact of violations of the sufficient conditions detailed in \autoref{sec:random_is_enough}. We report the robust test accuracy (against \textsc{AA+MT}~\citep{gowal_uncovering_2020}) when training different model architectures against $\epsilon_\infty = 8/255$. In panel \subref{fig:non_robust_accuracy_v2}, non-robust classifiers with different clean acccuracies are used for pseudo-labeling. In panel \subref{fig:biggan_generated_ratio}, we vary the mixture of training samples from a class-conditional Gaussian and a BigGAN distribution while the test distribution is the BigGAN distribution. In panel \subref{fig:biggan_set_coverage}, we fix the proportion of samples from the class-conditional Gaussian to 99\% and increase the number of classes from the BigGAN distribution seen during training (thus increasing coverage).
\label{fig:conditions}}
\end{figure*}

\subsection{Sufficient conditions}

The first set of experiments probes how violations of \autoref{cond:accuracy}, \ref{cond:gap} and \ref{cond:coverage} impact robustness against $\epsilon_\infty = 8/255$ (violations to \autoref{cond:unlikely} have an impact equivalent to those of \autoref{cond:accuracy}).
All experiments are summarized in \autoref{fig:conditions} where we train models with increasing capacity.

\paragraph{Non-robust classifier accuracy.}

In \autoref{fig:non_robust_accuracy_v2}, we train models using 1M samples generated by the \gls{ddpm} and vary the accuracy of the the pre-trained, non-robust classifier $\fnr$.
We evaluate the robust accuracy obtained on \cifar test set.
We observe that robustness improves as the accuracy of $\fnr$ increases.
Notably, even with the 74.47\%-accurate non-robust classifier, the \wrn-28-10 and \wrn-70-16 obtain robust accuracies of 58.15\% and 59.83\%, respectively, and already improve upon the state-of-the-art (57.14\% at the time of writing).
Thus, validating that, in practice, it is not necessary to have access to a perfect non-robust classifier.

\paragraph{Quality of the generative models.}

To analyze how the quality of the generative model influences robustness, we use the BigGAN to model the ``true'' data distribution.
During training, we use samples generated from a mixture of the class-conditional Gaussian and BigGAN distributions; during testing, we evaluate on a separate subset of 10K unseen BigGAN samples.
To probe \autoref{cond:gap}, we change the proportion of training samples from the class-conditional Gaussian.
In effect, decreasing the proportion of such samples skews the mixed generative model (modeled by the mixture of Gaussian and BigGAN distributions) to produce more samples from the true distribution (modeled by the BigGAN distribution), thereby closing the gap between the approximated distribution $\smash{\generateddata}$ and true distribution $\realdata$.
As expected, \autoref{fig:biggan_generated_ratio} demonstrates that, given enough capacity, models can significantly reduce the adversarial risk.

\vspace{-.3cm}
\paragraph{Relationship between coverage and capacity.}

Similarly to \autoref{fig:biggan_generated_ratio}, we use samples generated from a mixture of the class-conditional Gaussian and BigGAN distributions; during testing, we evaluate on a separate subset of 10K unseen BigGAN samples.
To probe \autoref{cond:coverage}, we keep the proportion of samples from the class-conditional Gaussian distribution fixed at 99\% and use the remaining 1\% to include BigGAN samples corresponding to either 0, 1, \ldots or 10 classes (thereby increasing coverage).
In other words, the coverage of the true data distribution $\realdata$ (given by the BigGAN) increases as the number of seen classes increases.
However, the approximated distribution $\smash{\generateddata}$ remains different from the true data distribution even when the coverage reaches all classes (as the proportion of Gaussian samples is fixed to 99\%).
We observe in \autoref{fig:biggan_set_coverage} that the robust accuracy of models with lower capacity improves less drastically -- yielding a gap of 17.37\% at full coverage between the ResNet-18 and \wrn-70-16 models.
This observation confirms that, with enough coverage, model capacity can compensate for the lack of a perfect generative model.

\vspace{-.3cm}
\paragraph{Discussion.}
Overall, \autoref{fig:biggan_generated_ratio} shows that when \autoref{cond:gap} is satisfied, the difference between models reduces and capacity takes a secondary role (since all models can bring their adversarial risk close to zero). \autoref{fig:biggan_set_coverage} shows that when \autoref{cond:coverage} is satisfied (and \autoref{cond:gap} is not), capacity matters as we observe that larger models benefit more from increased coverage.
Both figures point to the fact that the quality of the generative model becomes less important when the capacity of the classifiers increases (as long as coverage is sufficient).

\begin{figure}[t]
\centering
  \begin{minipage}[c]{0.45\textwidth}
    \includegraphics[width=\textwidth]{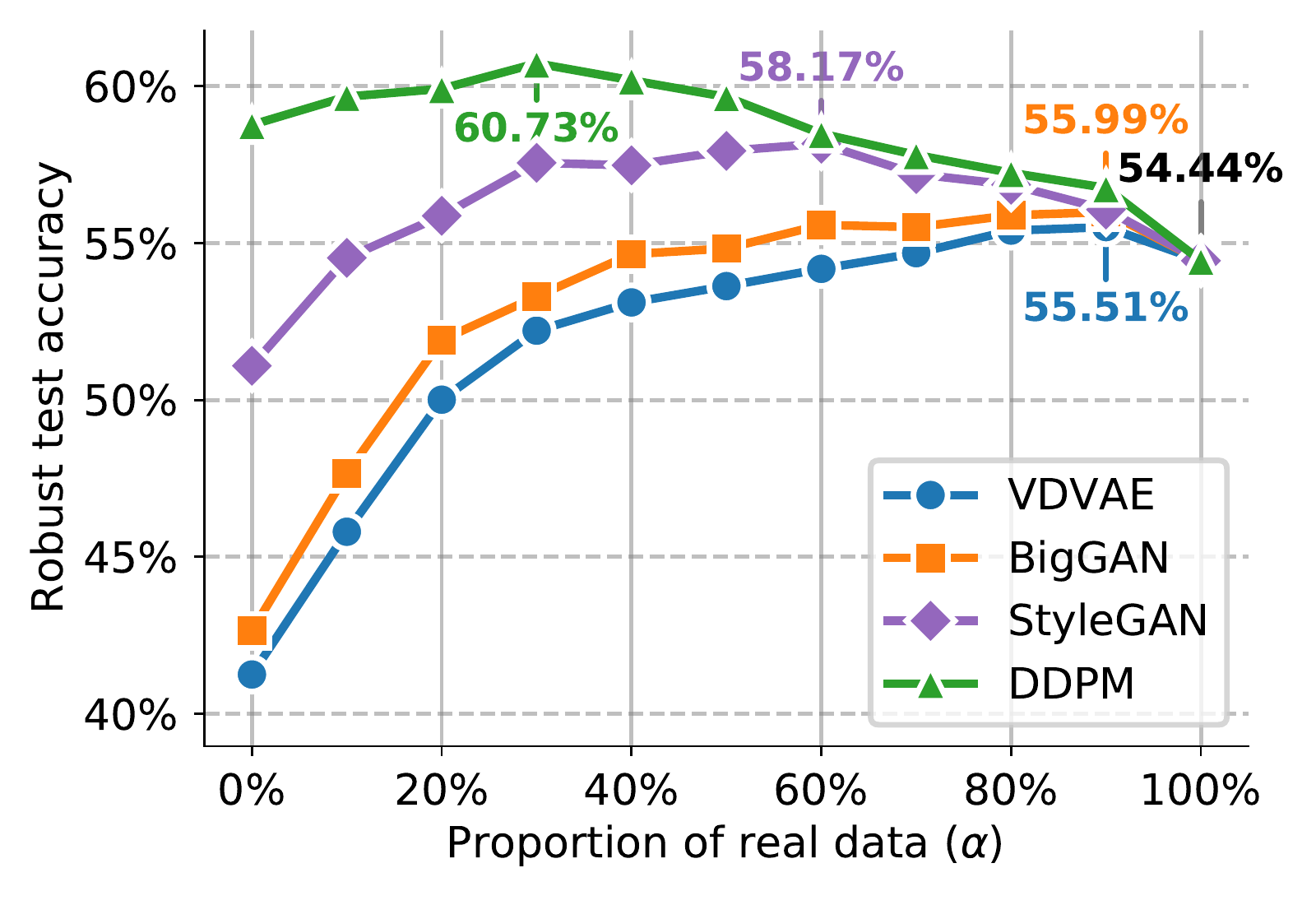}
  \end{minipage}\hfill
  \begin{minipage}[c]{0.5\textwidth}
    \caption{Robust test accuracy obtained by training a \wrn-28-10 against $\epsilon_\infty = 8/255$ on \cifar when using additional data produced by different generative models. We compare how the ratio between original and generated images (i.e., $\alpha$) affects robustness (0\% means generated samples only, 100\% means \cifar train set only). \label{fig:generative_ratio}}
  \end{minipage}
\end{figure}

\begin{table}[t]
\centering
\vspace{-.5cm}
\begin{minipage}[c]{0.38\textwidth}
    \caption{Clean (without perturbations) and robust (under adversarial attack) accuracy obtained by different models (we pick the worst accuracy obtained by either \autoattack or \textsc{AA+MT}). The accuracies are reported on the full test sets. For \cifar, we test against $\epsilon_\infty = 8/255$ and $\epsilon_2 = 128/255$. For \cifarh, \svhn and \tinyimagenet, we test against $\epsilon_\infty = 8/255$. For \imagenet, we test against $\epsilon_\infty = 4/255$. * This model is trained for 2000 epochs on 100M samples. ** This model is trained for 200 epochs on 5M samples.
  \label{table:results}}
  \end{minipage}\hfill
  \begin{minipage}[c]{0.6\textwidth}
\resizebox{\textwidth}{!}{
\begin{tabular}{l|cc|cc}
    \hline
    \cellcolor{header} \textsc{Model} & \cellcolor{header} \textsc{Dataset} & \cellcolor{header} \textsc{Norm} & \cellcolor{header} \textsc{Clean} & \cellcolor{header} \textsc{Robust} \TBstrut \\
    \hline
    \citet{wu2020adversarial} (\wrn-34-10) & \multirow{4}{*}{\cifar} & \multirow{4}{*}{\linf} & 85.36\% & 56.17\% \Tstrut \\
    \citet{gowal_uncovering_2020} (\wrn-70-16)  & & & 85.29\% & 57.14\% \\
    Ours (\gls{ddpm}) (\wrn-28-10) & & & 85.97\% & 60.73\% \\
    Ours (\gls{ddpm}) (\wrn-70-16) & & & 86.94\% & 63.58\% \\
    Ours (100M \gls{ddpm})* (ResNet-18) & & & 87.35\% & 58.50\% \\
    Ours (100M \gls{ddpm})* (\wrn-28-10) & & & 87.50\% & 63.38\% \\
    Ours (100M \gls{ddpm})* (\wrn-70-16) & & & \textbf{88.74\%} & \textbf{66.10\%} \Bstrut \\
    \hline
    \citet{wu2020adversarial} (\wrn-34-10) & \multirow{4}{*}{\cifar} & \multirow{4}{*}{\ltwo} & 88.51\% & 73.66\% \Tstrut \\
    \citet{gowal_uncovering_2020} (\wrn-70-16)  & & & \textbf{90.90\%} & 74.50\% \\
    Ours (\gls{ddpm}) (\wrn-28-10) & & & 90.24\% & 77.37\% \\
    Ours (\gls{ddpm}) (\wrn-70-16) & & & 90.83\% & \textbf{78.31\%} \Bstrut \\
    \hline
    \citet{cui2020learnable} (\wrn-34-10) & \multirow{4}{*}{\cifarh} & \multirow{4}{*}{\linf} & 60.64\% & 29.33\% \Tstrut \\
    \citet{gowal_uncovering_2020} (\wrn-70-16) & & & \textbf{60.86\%} & 30.03\% \\
    Ours (\gls{ddpm}) (\wrn-28-10) & & & 59.18\% & 30.81\% \\
    Ours (\gls{ddpm}) (\wrn-70-16) & & & 60.46\% & \textbf{33.49\%} \Bstrut \\
    \hline
    Ours (without \gls{ddpm}) (\wrn-28-10) & \multirow{2}{*}{\svhn} & \multirow{2}{*}{\linf} & 92.87\% & 56.83\% \Tstrut \\
    Ours (\gls{ddpm}) (\wrn-28-10) & & & \textbf{94.15\%} & \textbf{60.90\%} \Bstrut \\
    \hline
    Ours (without \gls{ddpm}) (\wrn-28-10) & \multirow{2}{*}{\tinyimagenet} & \multirow{2}{*}{\linf} & 51.56\% & 21.56\% \Tstrut \\
    Ours (\gls{ddpm}) (\wrn-28-10) & & & \textbf{60.95\%} & \textbf{26.66\%} \Bstrut \\
    \hline
    Ours (without \gls{ddpm})** (ResNet-152) & \multirow{2}{*}{\imagenet} & \multirow{2}{*}{\linf} & 68.74\% & 44.44\% \Tstrut \\
    Ours (\gls{ddpm})** (ResNet-152) & & & \textbf{68.78\%} & \textbf{45.30\%} \Bstrut \\
    \hline
\end{tabular}
}
  \end{minipage}
\end{table}

\subsection{State-of-the-art robust accuracy}

\paragraph{Effect of mixing factor ($\alpha$).}

As done in \autoref{sec:method}, we vary the proportion $\alpha$ of original images in each batch for all generated datasets.
\autoref{fig:generative_ratio} explores a wide range of proportions while training a \wrn-28-10 against $\epsilon_\infty = 8/255$ on \cifar.
Samples from all models improve robustness when mixed optimally, but only samples from the StyleGAN2 and \gls{ddpm} improve robustness significantly (+3.73\% and +6.29\%, respectively).
It is also interesting to observe that, in the case of the \gls{ddpm}, using 1M generated images is better than using the 50K images from the original train set only.
While this may seem surprising, it can easily be explained if we assume that the \gls{ddpm} produces many more high-quality, high-diversity images than the limited set of images present in the original data (c.f.~\cite{schmidt_adversarially_2018}).
We also observe that the optimal mixing factor is different for different generative models.
Indeed, increasing $\alpha$ reduces the gap to the true data distribution at the cost of less complementarity with the original train set (see \autoref{sec:impact_alpha}).

\paragraph{\cifar.}

Table~\ref{table:results} shows the performance of models trained with 1M samples generated by the \gls{ddpm} on \cifar against $\epsilon_\infty = 8/255$ and $\epsilon_2 = 128/255$.
Irrespective of their size, models trained with 1M \gls{ddpm} samples surpass the current state-of-the-art in robust accuracy by a large margin (+6.44\% and +3.81\%).
When using 100M \gls{ddpm} samples (and training for 2000 epochs), we reach 66.10\% robust accuracy against $\epsilon_\infty = 8/255$ which constitutes an improvement of +8.96\% over the state-of-the-art.
In this setting, our smallest model (ResNet-18) surpasses state-of-the-art results obtained by much larger models (e.g., \wrn-70-16).
Most remarkably, \uline{despite not using any external data}, against $\epsilon_\infty = 8/255$, our best model beats all RobustBench~\cite{croce2020robustbench} entries that used external data (see \autoref{table:robustbench} in the appendix).

\paragraph{Generalization to other datasets (\cifarh, \svhn, \tinyimagenet and \imagenet).}

Finally, to evaluate the generality of our approach, we evaluate it on \cifarh, \svhn~\cite{netzer2011reading}, \tinyimagenet~\cite{tinyimagenet} and \imagenet~\cite{krizhevsky_imagenet_2012}.
We train two new \gls{ddpm} on the train set of \cifarh and \svhn and sample 1M images from each.
For \tinyimagenet and \imagenet, we use a \gls{ddpm} trained on \imagenet~\cite{dhariwal2021diffusion} at $64\times64$ and $256\times256$ resolutions.
For \tinyimagenet, we restrict samples to the valid set of 200 classes.
The results are shown in Table~\ref{table:results}.
On \cifarh, our best model reaches a robust accuracy of 33.49\% and improves noticeably upon the state-of-the-art by +3.46\% (in the setting that does not use any external data).
On \svhn, in the same table, we compare models trained without and with \gls{ddpm} samples.
Again, the addition of \gls{ddpm} samples significantly improves robustness, with the robust accuracy improving by by +4.07\%.
On \tinyimagenet, the improvement is +5.10\%.
On \imagenet, the improvement is +0.86\% when training for 200 epochs (see \autoref{sec:additional_results} for more details).

\section{Conclusion.}

Using generative models, we posit and demonstrate that generated samples provide a greater diversity of augmentations that allow adversarial training to go well beyond the current state-of-the-art.
Our work provides novel insights into the effect of diversity and complementarity on robustness, which we hope can further our understanding of robustness.
All our models and generated datasets are available online at \githuburl.

\bibliography{bibliography}
\bibliographystyle{abbrvnat}

\clearpage
\appendix

\section{Experimental setup}
\label{sec:experimental_setup}

The implementation of the following setup is written in JAX~\citep{bradbury_jax_2018} and Haiku~\citep{hennigan_haiku_2020}.

\paragraph{Architecture.}

We use Residual Networks (ResNets) and Wide ResNets ({\wrn}s)~\citep{he2015deep,zagoruyko2016wide}.
This is consistent with prior work \citep{madry_towards_2017,rice_overfitting_2020,zhang_theoretically_2019,uesato_are_2019,gowal_uncovering_2020} which use diverse variants of these network families.
Furthermore, we adopt the same architecture details as \citet{gowal_uncovering_2020} with Swish/SiLU~\citep{hendrycks2016gaussian} activation functions.
Most of the experiments are conducted on a \wrn-28-10 model which has a depth of 28, a width multiplier of 10 and contains 36M parameters.
To evaluate the effect of using additional generated data on wider and deeper networks, we also run several experiments using \wrn-70-16, which contains 267M parameters.

\paragraph{Outer minimization.}
We use TRADES~\citep{zhang_theoretically_2019} optimized using SGD with Nesterov momentum~\citep{polyak1964some, nesterov27method} and a global weight decay of $5 \times 10^{-4}$. 
We use a batch size of $1024$ split over $32$ Google Cloud TPUv3 cores~\citep{kumar2019scale}, train for $400$ \cifar-equivalent epochs (resulting in $19$K training steps), and use a \emph{cosine} learning rate schedule~\citep{SGDR} without restarts where the initial learning rate is set to 0.4 and is decayed to 0 by the end of training (similar to \citep{gowal_uncovering_2020}).
We also use model weight averaging (WA)~\citep{izmailov_averaging_2018} with a decay rate of $\tau=0.995$.
With this setup, training a \wrn-28-10, a \wrn-70-16 and a ResNet-18 takes 2.5 hours, 6 hours and 22 minutes, respectively.

\paragraph{Inner minimization.}
Adversarial examples are obtained by maximizing the  Kullback-Leibler divergence between the predictions made on clean inputs and those made on adversarial inputs~\citep{zhang_theoretically_2019}.
This optimization procedure is done using the Adam optimizer~\citep{kingma_adam:_2014} with a step-size of $0.1$ and $10$ steps.

\paragraph{Evaluation.}

We follow the evaluation protocol designed by \citet{gowal_uncovering_2020}.
Specifically, we train two (and only two) models for each hyperparameter setting, perform early stopping for each model on a separate validation set of 1024 samples using \pgd{40} (i.e., \gls{pgd} with 40 gradient ascent steps) similarly to~\citet{rice_overfitting_2020} and pick the best model by evaluating the robust accuracy on the same validation set.
The average absolute difference between these two models is -0.12\% in test robust accuracy (as measured over 10 separate runs).
Unless stated otherwise, we always report the robust test accuracy against a mixture of \autoattack~\citep{croce_reliable_2020} and \multitargeted~\citep{gowal_alternative_2019}, which is denoted by \textsc{AA+MT}.
This mixture consists in completing the following sequence of attacks: \autopgd on the cross-entropy loss with 5 restarts and 100 steps, \autopgd on the difference of logits ratio loss with 5 restarts and 100 steps and finally \multitargeted on the margin loss with 10 restarts and 200 steps.
We note that, while early stopping is not necessary when using the cosine learning rate schedule, we keep it to be consistent with prior work.

\section{Additional results}
\label{sec:additional_results}

\begin{figure}[b]
\centering
  \begin{minipage}[c]{0.68\textwidth}
    \includegraphics[width=\textwidth]{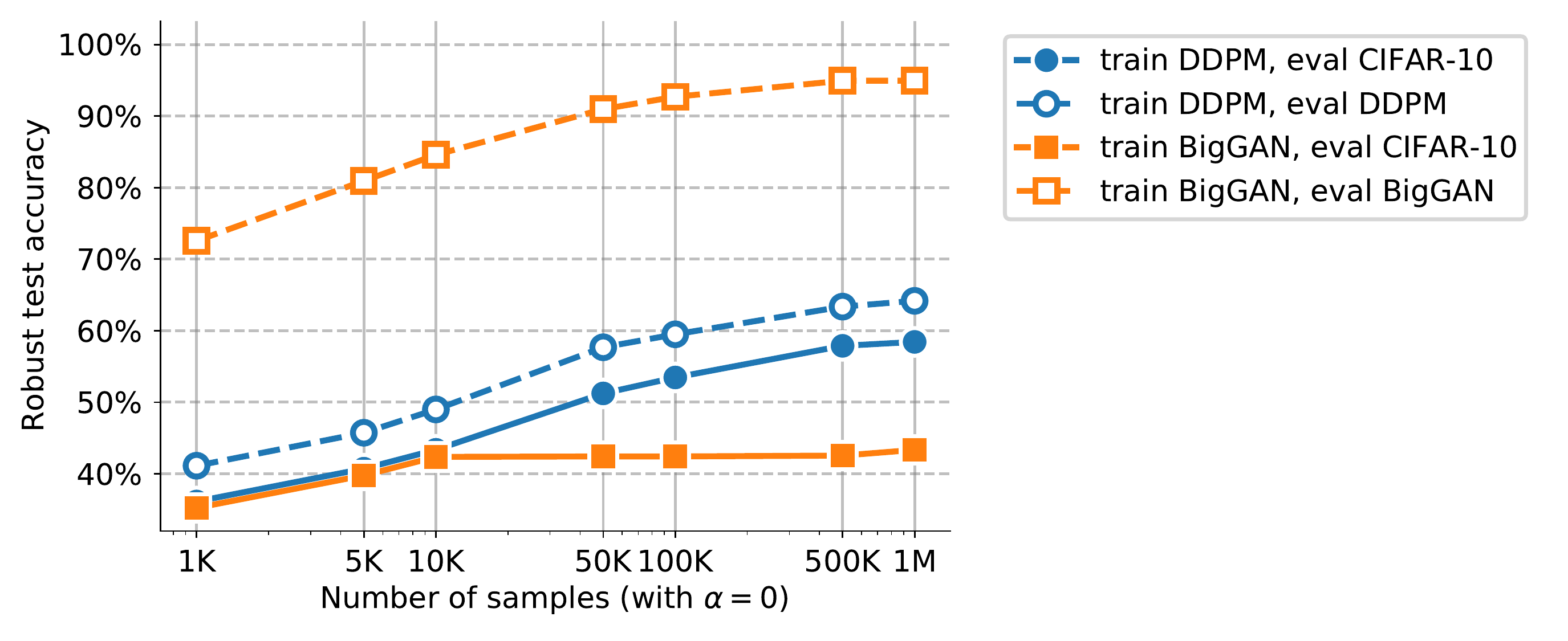}
  \end{minipage}\hfill
  \begin{minipage}[c]{0.32\textwidth}
    \caption{Robust test accuracy when training a \textsc{Wrn}-28-10 using a variable number of samples from a \gls{ddpm} or BigGAN. We compare the robust accuracy on the \cifar test set with the one obtained on a separate set of generated samples. \label{fig:num_samples}}
  \end{minipage}
\end{figure}

\paragraph{Scaling dataset size.}

Using a generative model allows us to sample many more images than available in the original training set.
In \autoref{fig:num_samples}, we set the mixing factor $\alpha$ to zero (thus only using generated samples) and vary the number of training samples.
We evaluate the robust accuracy of the resulting model on the \cifar test set and on a separate validation set composed of 10K generated samples.
We also compare models trained on BigGAN and \gls{ddpm} samples.
Irrespective of the underlying generative model, using more samples generally improves robustness.
Samples from the \gls{ddpm} are more useful, as can be seen from the higher robust accuracy obtained on the \cifar test set (i.e., 58.43\% versus 43.34\% with a \wrn-28-10).
It is also worth noting that using samples from the \gls{ddpm} results in a smaller generalization gap of 5.74 points when using 1M samples (gap between the dashed and solid blue lines).
Models trained on BigGAN samples tend to overfit to these samples, which results in a large generalization gap of 51.61 points (gap between the dashed and solid orange lines).
These results also confirm that BigGAN samples are easier to robustly classify (possibly due to their low diversity).

\paragraph{\cifarh.}

For completeness, we also report the effect of mixing different proportions of generated and original samples in \autoref{fig:cifarh_ratio} against $\epsilon_\infty = 8/255$ using a \wrn-28-10 on \cifarh.
Similarly to \autoref{fig:generative_ratio}, we observe that additional samples generated by \gls{ddpm} are useful to improve robustness, with an absolute improvement of +2.48\% in robust accuracy.

\begin{figure}[t]
\centering
\begin{minipage}{.45\textwidth}
  \centering
  \includegraphics[width=\textwidth]{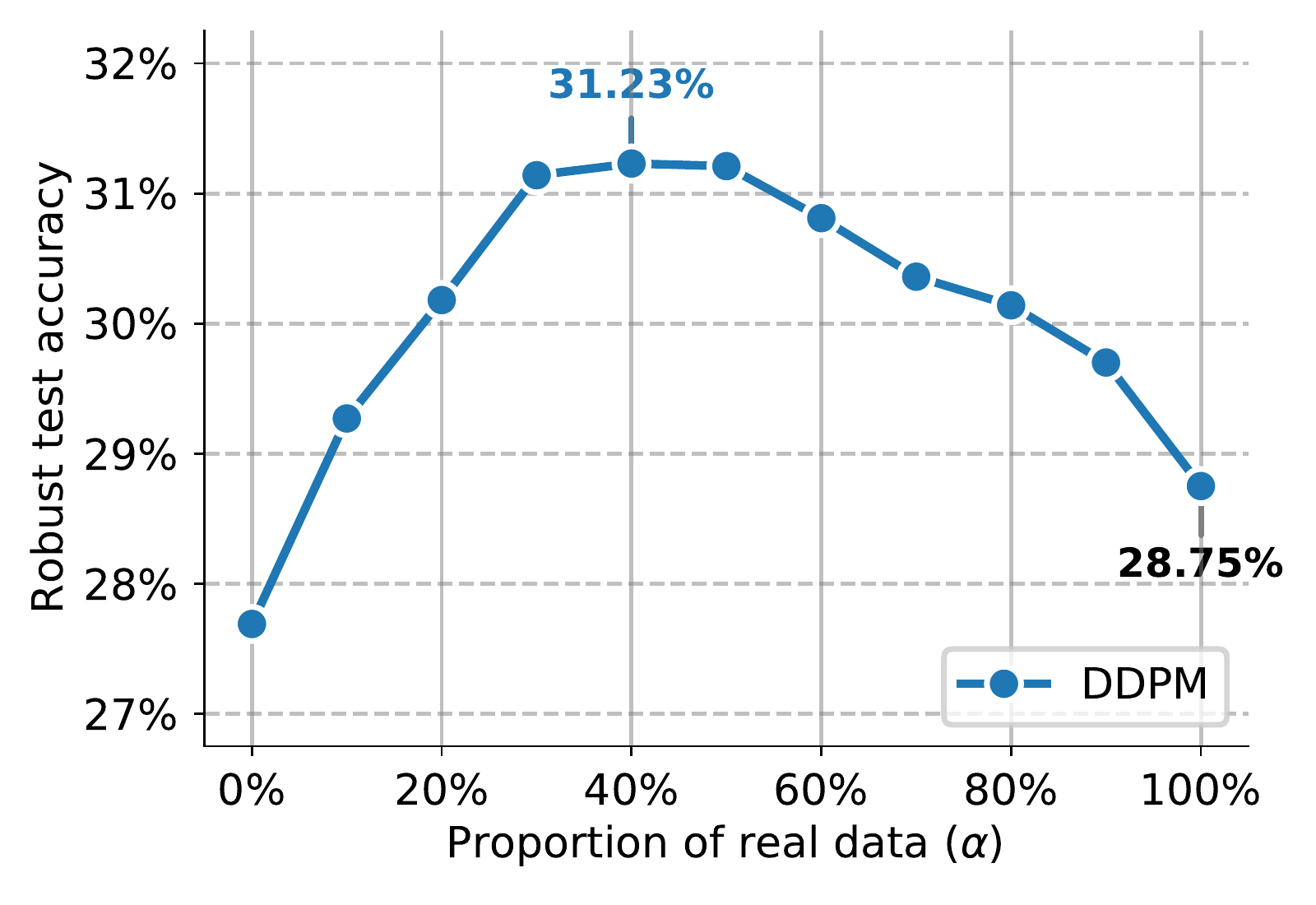}
  \captionof{figure}{Robust test accuracy when training a \wrn-28-10 against $\epsilon_\infty = 8/255$ on \cifarh with additional data produced by a \gls{ddpm}. We compare how the ratio between original images and generated images in the training minibatches affects the test robust performance (0\% means generated samples only, while 100\% means original \cifarh train set only). \label{fig:cifarh_ratio}}
\end{minipage}\hfill
\begin{minipage}{.45\textwidth}
  \centering
  \includegraphics[width=\textwidth]{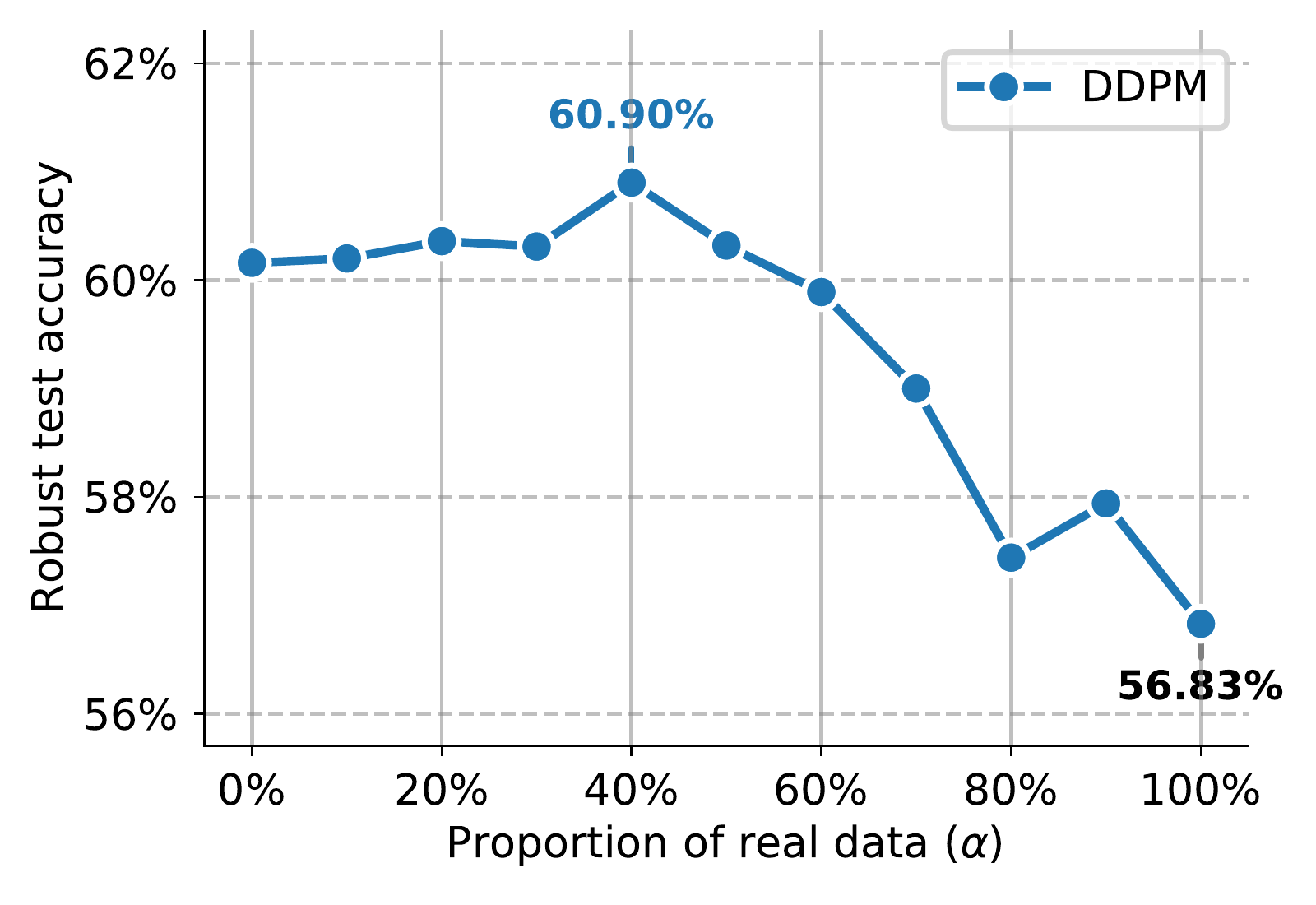}
  \captionof{figure}{Robust test accuracy when training a \wrn-28-10 against $\epsilon_\infty = 8/255$ on \svhn with additional data produced by a \gls{ddpm}. We compare how the ratio between original images and generated images in the training minibatches affects the test robust performance (0\% means generated samples only, while 100\% means original \svhn train set only). \label{fig:svhn_ratio}}
\end{minipage}
\end{figure}

\begin{figure}[t]
\centering
  \begin{minipage}[c]{0.45\textwidth}
    \includegraphics[width=\textwidth]{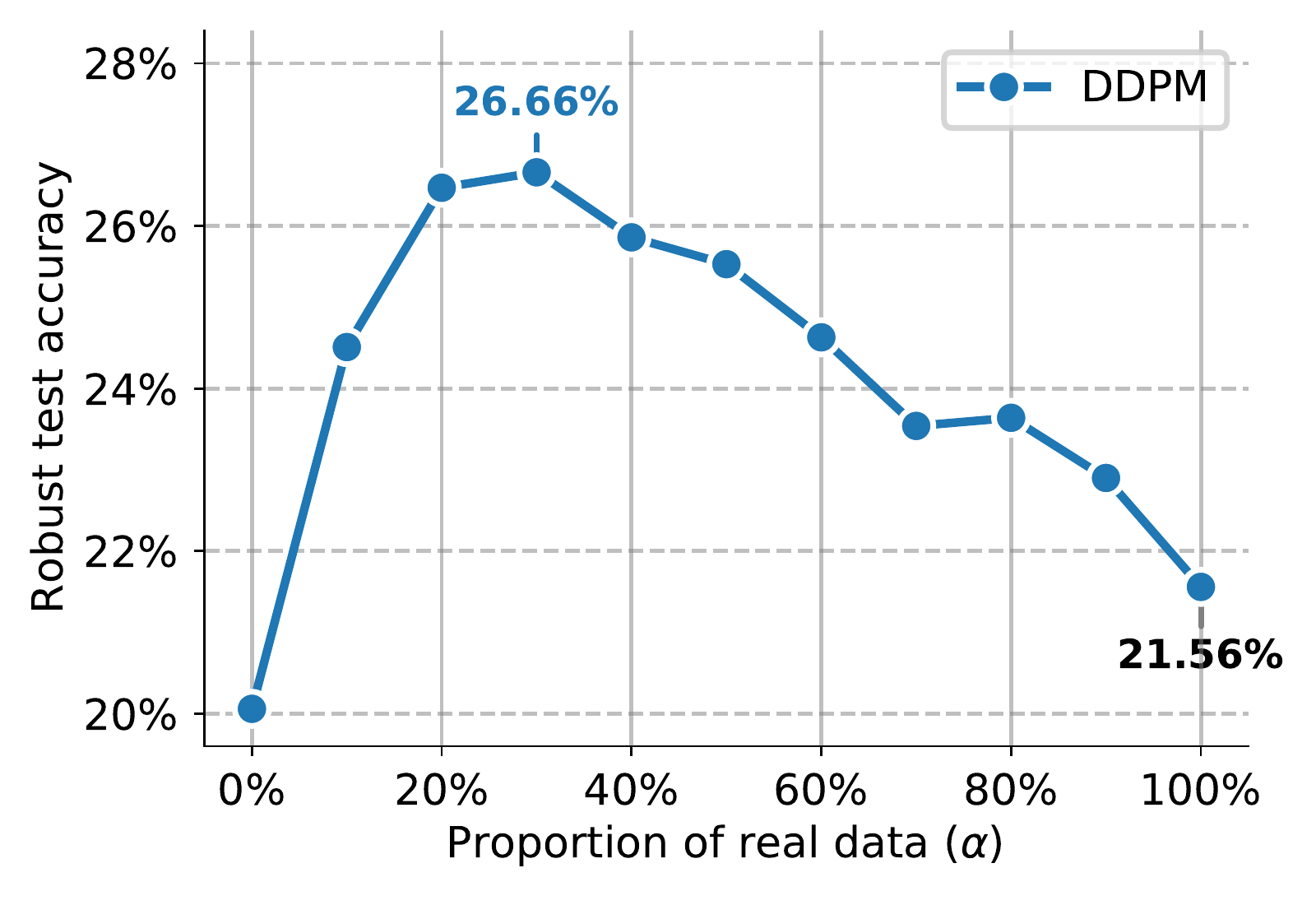}
  \end{minipage}\hfill
  \begin{minipage}[c]{0.45\textwidth}
    \caption{Robust test accuracy when training a \wrn-28-10 against $\epsilon_\infty = 8/255$ on \tinyimagenet with additional data produced by a \gls{ddpm} trained on \imagenet. We compare how the ratio between original images and generated images in the training minibatches affects the test robust performance (0\% means generated samples only, while 100\% means original \tinyimagenet train set only). \label{fig:tinyimagenet_ratio}}
  \end{minipage}
\end{figure}

\paragraph{\svhn.}

We report the effect of mixing different proportions of generated and original samples in \autoref{fig:svhn_ratio} against $\epsilon_\infty = 8/255$ using a \wrn-28-10 on \svhn.
Similarly to \autoref{fig:generative_ratio} and \autoref{fig:cifarh_ratio}, we observe that additional samples generated by \gls{ddpm} are useful to improve robustness, with an absolute improvement of +4.07\% in robust accuracy.

\paragraph{\tinyimagenet.}

We report the effect of mixing different proportions of generated and original samples in \autoref{fig:tinyimagenet_ratio} against $\epsilon_\infty = 8/255$ using a \wrn-28-10 on \tinyimagenet.
Similarly to \autoref{fig:generative_ratio}, \autoref{fig:cifarh_ratio} and \autoref{fig:svhn_ratio}, we observe that additional samples generated by \gls{ddpm} are useful to improve robustness, with an absolute improvement of +5.10\% in robust accuracy.

\paragraph{\imagenet.}

To obtain the results in \autoref{table:results} for \imagenet, we train a ResNet-152 with SiLU/Swish activations, with and without 5M additional DDPM samples against $\epsilon_\infty = 4/255$.
We use standard adversarial training \citep{madry_towards_2017} for 200 epochs with cosine learning rate schedule and weight averaging ($\tau = .995)$ and set the proportion $\alpha$ of real data to 80\%.
We use a batch size of $1024$ split over $128$ Google Cloud TPUv3 cores~\citep{kumar2019scale}.
The robust accuracy reported is obtained by running \textsc{AA+MT} where we limit the targeted attacks to the top-10 non-correct class predictions.
For 200 epochs of training, we obtain 44.44\% and 45.30\% without and with 5M additional DDPM samples, respectively.
When training for 100 epochs, we obtain 40.04\% and 42.31\% without and with 5M additional DDPM samples, respectively (with 64.79\% and 66.09\% in clean accuracy).
It is worth noting that when training a ResNet-50, DDPM samples do not provide a significant improvement -- with both models reaching 42.57\% and 42.75\% robust accuracy after 200 epochs without and with DDPM samples, respectively (with 67.40\% and 67.84\% in clean accuracy).
This observation is in line with the theory elaborated in \autoref{sec:conditions}.

\paragraph{Clean accuracy.}

Finally, the clean accuracy (i.e., accuracy obtained when no perturbation is applied to the input) of all models used in \autoref{fig:generative_ratio} is reported in \autoref{fig:clean_accuracy}.
All these models are trained adversarially to be robust against $\epsilon_\infty = 8/255$ on \cifar.
We observe that improvements in robust accuracy are not always correlated (either positively or negatively) with improvements in clean accuracy.
While \gls{vdvae} samples provide no improvements in clean accuracy, using BigGAN, StyleGAN2 or \gls{ddpm} samples can improve clean accuracy by up to +1.27\%, +3.45\% and +2.05\%, respectively.

\begin{figure}[t]
\centering
  \begin{minipage}[c]{0.45\textwidth}
    \includegraphics[width=\textwidth]{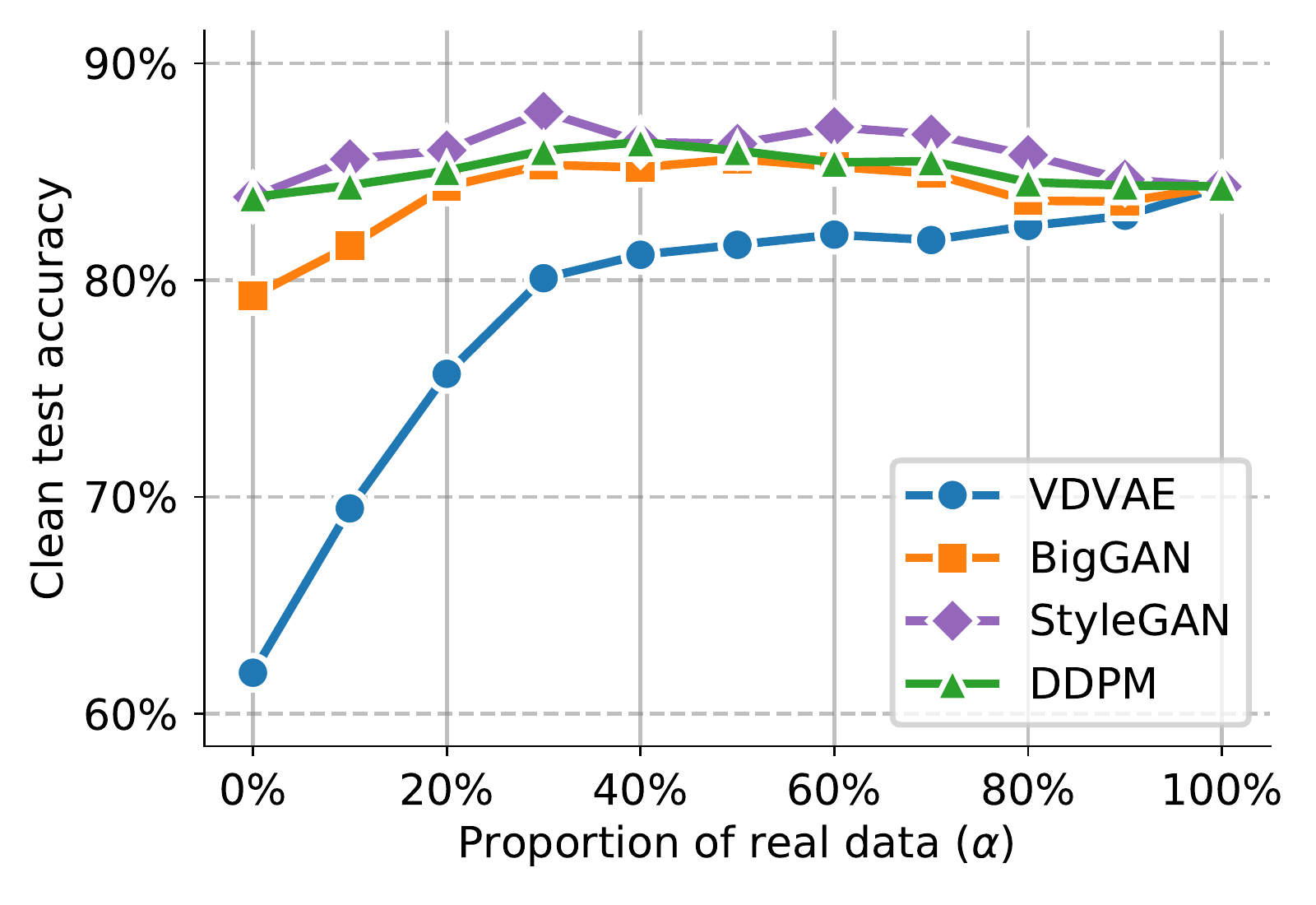}
  \end{minipage}\hfill
  \begin{minipage}[c]{0.45\textwidth}
    \caption{Clean test accuracy obtained by training a \wrn-28-10 against $\epsilon_\infty = 8/255$ on \cifar when using additional data produced by different generative models. We compare how the ratio between original and generated images (i.e., $\alpha$) affects the clean accuracy (0\% means generated samples only, 100\% means \cifar train set only). The robust test accuracy for the same models is shown in \autoref{fig:generative_ratio} in the main manuscript. \label{fig:clean_accuracy}}
  \end{minipage}
\end{figure}

\section{Analysis of models}

In this section, we perform additional diagnostics that give us confidence that our models are not doing any form of gradient obfuscation or masking \citep{athalye_obfuscated_2018,uesato_adversarial_2018}.

\paragraph{\autoattack and robustness against black-box attacks.}

First, we report in \autoref{table:autoattack} the robust accuracy obtained by our strongest models against a diverse set of attacks.
These attacks are run as a cascade using the \autoattack library available at \url{https://github.com/fra31/auto-attack}.
First, we observe that our combination of attacks, denoted \textsc{AA+MT} matches the final robust accuracy measured by \autoattack.
Second, we also notice that the black-box attack (i.e., \textsc{Square}) does not find any additional adversarial examples.
Overall, these results indicate that our empirical measurement of robustness is meaningful and that our models do not obfuscate gradients.

\begin{table}[h]
\caption{Clean (without adversarial attacks) accuracy and robust accuracy (against the different stages of \autoattack) on \cifar obtained by different models. Refer to \url{https://github.com/fra31/auto-attack} for more details.\label{table:autoattack}}
\begin{center}
\resizebox{1.\textwidth}{!}{
\begin{tabular}{l|ccc|cccc|cc}
    \hline
    \cellcolor{header} \textsc{Model} & \cellcolor{header} \textsc{Dataset} & \cellcolor{header} \textsc{Norm} & \cellcolor{header} \textsc{Radius} & \cellcolor{header} \textsc{AutoPGD-ce} & \cellcolor{header} + \textsc{AutoPGD-t} & \cellcolor{header} + \textsc{Fab-t} & \cellcolor{header} + \textsc{Square} & \cellcolor{header} \textsc{Clean} & \cellcolor{header} \textsc{AA+MT} \TBstrut \\
    \hline
    \wrn-28-10 (\gls{ddpm}) & \multirow{3}{*}{\cifar} & \multirow{3}{*}{\linf} & \multirow{3}{*}{$\epsilon = 8/255$} & 63.53\% & 60.73\% & 60.73\% & 60.73\% & 85.97\% & 60.73\% \Tstrut \\
    \wrn-70-16 (\gls{ddpm}) & & & & 65.95\% & 63.62\% & 63.62\% & 63.62\% & 86.94\% & 63.58\% \\
    ResNet-18 (100M \gls{ddpm}) & & & & 60.85\% & 58.63\% & 58.63\% & 58.63\% & 87.35\% & 58.50\% \\
    \wrn-28-10 (100M \gls{ddpm}) & & & & 65.65\% & 63.44\% & 63.44\% & 63.44\% & 87.50\% & 63.38\% \\
    \wrn-70-16 (100M \gls{ddpm}) & & & & 68.46\% & 66.13\% & 66.11\% & 66.11\% & 88.74\% & 66.10\% \Bstrut \\
    \hline
    \wrn-28-10 (\gls{ddpm}) & \multirow{2}{*}{\cifar} & \multirow{2}{*}{\ltwo} & \multirow{2}{*}{$\epsilon = 128/255$} & 78.13\% & 77.44\% & 77.44\% & 77.44\% & 90.24\% & 77.37\% \Tstrut \\
    \wrn-70-16 (\gls{ddpm}) & & & & 78.97\% & 78.39\% & 78.39\% & 78.39\% & 90.93\% & 78.31\% \Bstrut \\
    \hline
    \wrn-28-10 (\gls{ddpm}) & \multirow{2}{*}{\cifarh} & \multirow{2}{*}{\linf} & \multirow{2}{*}{$\epsilon = 8/255$} & 34.47\% & 30.81\% & 30.81\% & 30.81\% & 59.18\% & 31.23\% \Tstrut \\
    \wrn-70-16 (\gls{ddpm}) & & & & 36.27\% & 33.49\% & 33.49\% & 33.49\% & 60.46\% & 33.93\% \Bstrut \\
    \hline
\end{tabular}
}
\end{center}
\end{table}

\paragraph{Loss landscapes.}

We analyze the adversarial loss landscapes of our best model trained on \cifar against $\epsilon_\infty = 8/255$ (a \wrn-70-16).
To generate a loss landscape, we vary the network input along the linear space defined by the worse perturbation found by \pgd{40} ($u$ direction) and a random Rademacher direction ($v$ direction).
The $u$ and $v$ axes represent the magnitude of the perturbation added in each of these directions respectively and the $z$ axis is the adversarial margin loss~\citep{carlini_towards_2017}: $z_y - \max_{i \neq y} z_i$ (i.e., a misclassification occurs when this value falls below zero).
\autoref{fig:linf_landscapes} shows the loss landscapes around the first 2 images of the \cifar test set for the aforementioned model.
Both landscapes are smooth and do not exhibit patterns of gradient obfuscation.
Overall, it is difficult to interpret these figures further, but they do complement the numerical analyses done so far.

\begin{figure*}[t]
\centering
\subfigure[Image of a horse]{\includegraphics[width=0.45\textwidth]{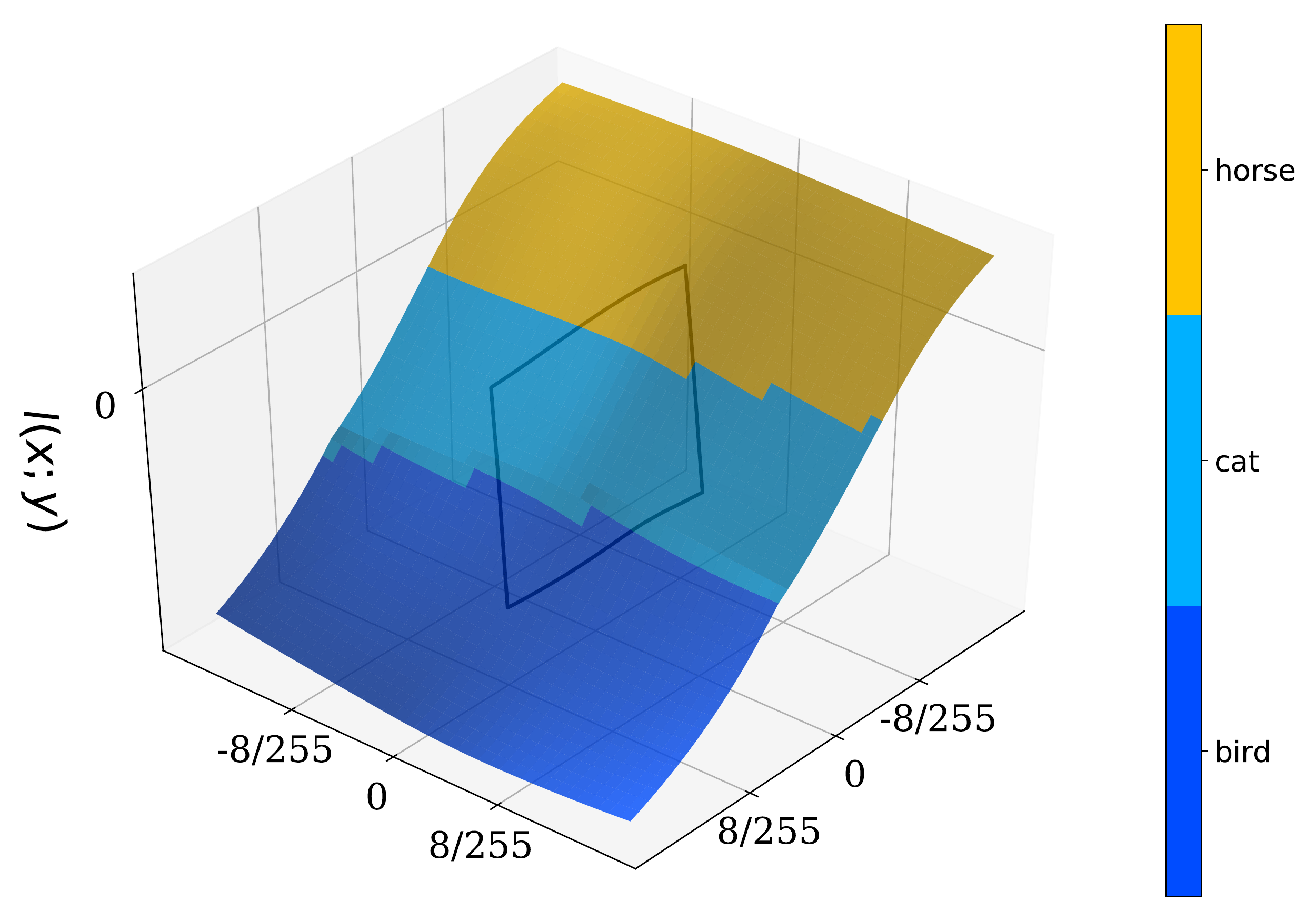}}
\subfigure[Image of an airplane]{\includegraphics[width=0.45\textwidth]{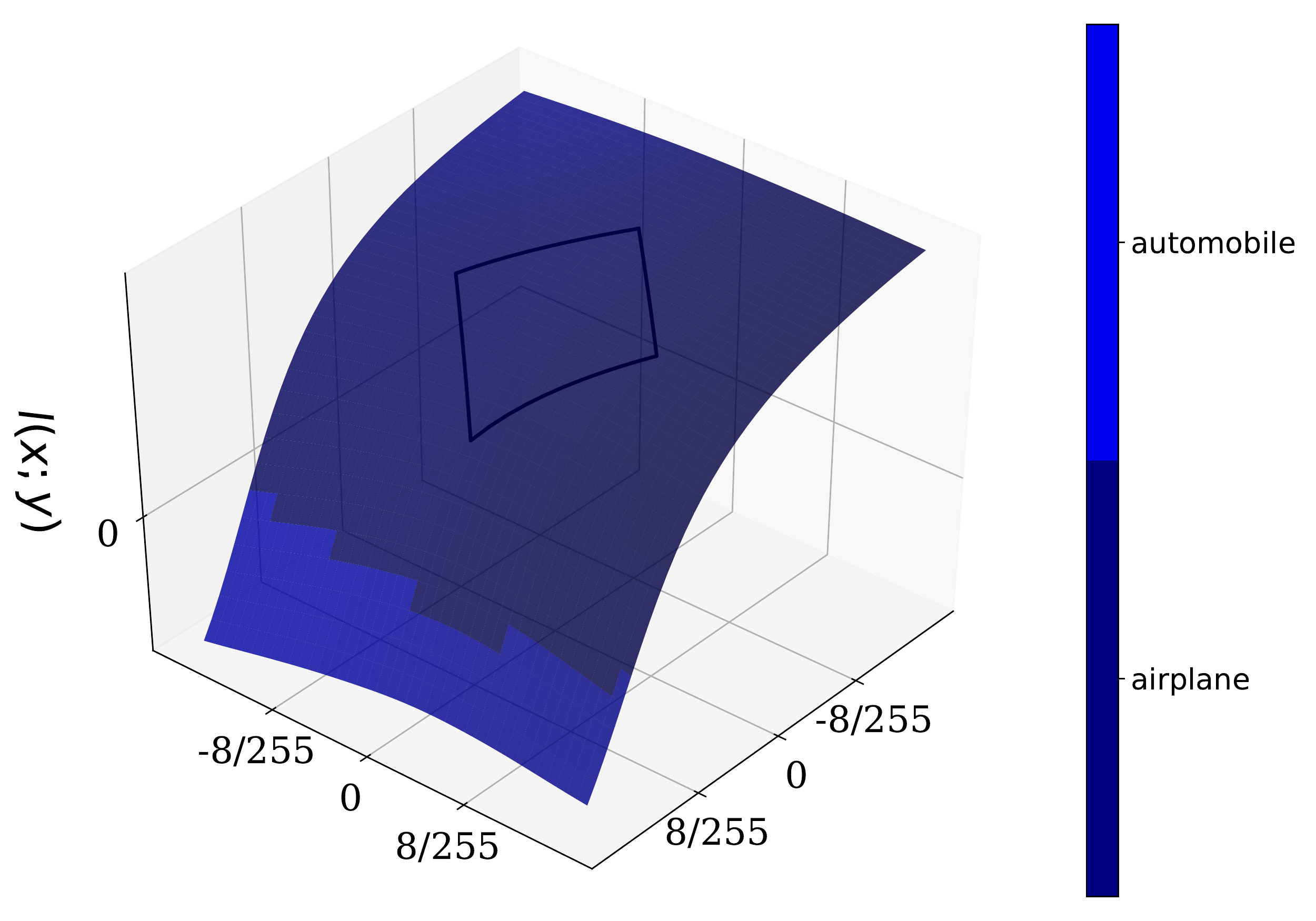}}
\caption{Loss landscapes around the first two images from the \cifar test set for the \wrn-70-16 networks trained with \gls{ddpm} samples.
It is generated by varying the input to the model, starting from the original input image toward either the worst attack found using \pgd{40} ($u$ direction) or a random Rademacher direction ($v$ direction). The loss used for these plots is the margin loss $z_y - \max_{i \neq y} z_i$ (i.e., a misclassification occurs when this value falls below zero). The diamond-shape represents the projected \linf ball of size $\epsilon = 8/255$ around the nominal image.}
\label{fig:linf_landscapes}
\end{figure*}

\section{Details on generated data}
\label{sec:details_data_augment}
\label{sec:generated_samples}

\paragraph{Generative models.}

In this paper, we use four different and complementary generative models: \textit{(i)} BigGAN~\citep{brock2018large}, \textit{(ii)} \gls{vdvae}~\cite{child2021vdvae}, \textit{(iii)} StyleGAN2~\cite{karras2020analyzing} and \textit{(iv)} \gls{ddpm}~\cite{ho2020denoising}.
Except for BigGAN, we use the \cifar checkpoints that are available online.
For BigGAN, we train our own model and pick the model that achieves the best \gls{fid} (the model architecture and training schedule is the same as the one used in \citep{brock2018large}).
All models are trained solely on the \cifar train set (or the train set of \cifarh or \svhn for experiments shown in \autoref{sec:additional_results}).
As a baseline, we also fit a class-conditional multivariate Gaussian, which reaches \gls{fid} and \gls{is} metrics of 117.62 and 3.64, respectively.
We also report that BigGAN reaches an \gls{fid} of 11.07 and \gls{is} of 9.71; \gls{vdvae} reaches an \gls{fid} of 36.88 and \gls{is} of 6.03; StyleGAN2 reaches an \gls{fid} of 2.57 and \gls{is} of 10.04 and \gls{ddpm} reaches an \gls{fid} of 3.15 and \gls{is} of 9.50.\footnote{For \cifarh, we trained our own \gls{ddpm} which achieves an \gls{fid} of 5.58 and \gls{is} of 10.82.}\footnote{For \svhn, we trained our own \gls{ddpm} which achieves an \gls{fid} of 4.89 and \gls{is} of 3.06.}\footnote{For \tinyimagenet, we used the class-conditional \gls{ddpm} checkpoint available at \url{https://github.com/openai/guided-diffusion} which has been trained on \imagenet at a $64\times64$ resolution.}

\paragraph{Datasets of generated samples.}

For the class-conditional multivariate Gaussian and \gls{ddpm} samples, we use a pretrained \wrn-28-10 to give pseudo-labels.
This \wrn-28-10 is trained non-robustly on the \cifar train set and achieves 95.68\% accuracy.\footnote{For \cifarh, the same model achieves 79.98\% accuracy.}\footnote{For \svhn, the same model achieves 96.54\% accuracy.}\footnote{For \tinyimagenet, we use a class-conditional StyleGAN2 model and do not need to train a non-robust classifier for pseudo-labeling.}
We sample images from each model until we have 100K images for each class.\footnote{We use 10K images per class for \cifarh experiments.}
For BigGAN and \gls{vdvae} samples, we proceed with an additional filtering step similarly to the one proposed by \citet{carmon_unlabeled_2019}.
We sample from each generative model 5M images and score each image using the pretrained, non-robust \wrn-28-10 model used for pseudo-labeling.
For each class, we select the top-100K scoring images and build a dataset of 1M image-label pairs.\footnote{All generated datasets are available online at \githuburl.}

This additional generated data (consisting of 1M samples) is used to train adversarially robust models by mixing in each batch a given proportion of original and generated examples.
\autoref{fig:samples} shows a random subset of this additional data for each generative model.
We also report the \gls{fid} and \gls{is} metrics of the resulting sets in \autoref{table:similarity_train_test_self_gen} and \autoref{table:similarity_train_test_self_gen_lpips}.
They might differ from metrics obtained by each generative model individually (see previous paragraph) as we filter images to either keep the highest scoring ones or make sure that classes are balanced.

\paragraph{Class-conditional Gaussian-fit.}

To generate images from the class-conditional Gaussian-fit, we use the following procedure: (i) we take 5K images for each class in the training set and do a principal component analysis (PCA) to reduce the dimensionality to 200 coordinates; (ii) in the reduced space, we compute the mean and covariance for each class; (iii) to generate an image from a given class, we sample from the corresponding Gaussian distribution (determined by its mean and covariance) and reverse the PCA transformation. The exact code is given below:
\begin{lstlisting}[language=Python]
import sklearn.decomposition
import numpy as np
import tensorflow_datasets as tfds

num_classes = 10
num_components = 200  # For PCA.

ds = tfds.load(name='cifar10', split=tfds.Split.TRAIN).batch(50_000)
data = next(iter(ds))
images = data['image'].numpy().astype(np.float32) / 255.
labels = data['label'].numpy()

class_pca = []
class_means = []
class_covs = []
for class_idx in range(num_classes):
  x = np.reshape(images[labels == class_idx], [-1, np.prod(images.shape[1:])])
  pca = sklearn.decomposition.PCA(n_components=num_components)
  y = pca.fit_transform(x)
  class_pca.append(pca)
  class_means.append(np.mean(y, axis=0))
  class_covs.append(np.cov(y.T))

# Sampling function.
def sample(class_idx, num_samples=1):
  x = np.random.multivariate_normal(
      class_means[class_idx], class_covs[class_idx], size=num_samples)
  y = class_pca[class_idx].inverse_transform(x)
  y = np.clip(y, 0, 1)
  return np.reshape(y, (num_samples,) + images.shape[1:])
\end{lstlisting}

\paragraph{Diversity and complementarity.}

While the \gls{fid} metric does capture how two distributions of samples match, it does not necessarily provide enough information in itself to assess the overlap between the distribution of generated samples and the train or test distributions (this is especially true for samples obtained through data augmentations such as \emph{mixup}) -- as seen in \autoref{table:similarity_train_test_self_gen} and explained in the next paragraph.
As such, we also decide to compute the proportion of nearest neighbors in perceptual space: given equal Inception metrics, a better generative model would produce samples that are equally likely to be close to training, testing or generated images.
In an attempt to estimate the coverage of the real data distribution, we also compute the proportion of nearest neighbors that are unique: given equal Inception metrics, a better generative model would produce samples that are equally likely to be close to any image in train or test set (thus resulting in a high proportion of unique neighbors).\footnote{As a point of comparison, sampling 2 sets of $10$K points from a uniform distribution $\mathcal{U}_{[0, 1]}$ between 0 and 1 yields in average a proportion of unique nearest neighbors equal to 55.6\%.}

We now describe how we compute \autoref{table:similarity_train_test_self_gen} which reports nearest-neighbors statistics for the different generative models.
First, we sample 10K images from the train set of \cifar (uniformly across classes) and take the full test set of \cifar.
We then pass these 20K images through the pre-trained Inception network (used to measure Inception metrics).
We use the activations from the last pooling operation and compute their top-100 PCA components, as this allows us to compare samples in a much lower dimensional space (i.e., 100 instead of 2048).
Finally, for each generative model, we sample 10K images from their 1M dataset (class-balanced as well) and pass them through the pipeline composed of the Inception network and the PCA projection computed on the original data.
The left-most three columns (entitled ``\uline{complementarity}'') are computed by finding, for each generated sample, its closest neighbor in the PCA-reduced feature space to any image from the set of $30\textrm{K}-1$ images composed of train, test and generated sets.
We then measure whether this nearest-neighbor belongs to the original datasets of 10K image each (train or test) or to the generated set (self) composed of the remaining $10\textrm{K}-1$ images.
For example, given 6 generated samples (instead of the 10K), the first sample's closest neighbor could be in the train set, the next two samples' closest neighbors could in the test set and the last three samples' closest neighbors could be in the set of generated samples. This would result in ratios of $1/6$, $1/3$ and $1/2$.
The middle set of two columns (entitled ``\uline{coverage}'') is computed by finding, for each generated sample, its closest neighbor in the PCA-reduced feature space to any image from the train and test sets.
We then measure the number of unique neighbors matched in both sets.\footnote{According to this measure, the \cifar train set covers 52.27\% of the test set, while the test set covers 52.08\% of the train set. Obtaining a significantly higher coverage of the train set is likely the result of overfitting and memorization.}
See Alg.~\ref{alg:neighbor} for pseudo-code.

\begin{algorithm}[h]
\footnotesize
\caption{Complementarity and coverage computation}
\begin{algorithmic}[1]
  \Require Train set $\train$, test set $\test$, distribution $\generateddata$ for which we measure complementary and coverage, number of samples $N$ and a function $g: \R^n \mapsto \R^m$ that maps inputs to their features (e.g., Inception features).
  \Ensure Complementarity $\{c_\textrm{train}, c_\textrm{test}, c_\textrm{self}\}$ and coverage $\{v_\textrm{train}, v_\textrm{test}\}$.
  \State $\gD_\textrm{self} \gets \{\vx_i \sim \generateddata \}_{i=1}^N$ \Comment{Pick $N$ samples from $\generateddata$}
  \State $\bar{\gD}_\textrm{train}$ is such that $\bar{\gD}_\textrm{train}\subseteq\train$ and $|\bar{\gD}_\textrm{train}| = N$ \Comment{Pick $N$ samples from $\train$}
  \State $\bar{\gD}_\textrm{test}$ is such that $\bar{\gD}_\textrm{test}\subseteq\train$ and $|\bar{\gD}_\textrm{test}| = N$ \Comment{Pick $N$ samples from $\test$}
  \State $c_\textrm{train} \gets 0, c_\textrm{test} \gets 0, c_\textrm{self} \gets 0$ \Comment{Initialize complementarity counters}
  \State $\gV_\textrm{train} \gets \emptyset, \gV_\textrm{test} \gets \emptyset$ \Comment{Initialize coverage sets}
  \For{$\vx_i \in \gD_\textrm{self}$} \Comment{For all generated samples}
    \State $\bar{\gD}_\textrm{self} = \gD_\textrm{self} \setminus \{\vx_i\}$ \Comment{Ignore current sample in computation below}
    \State $s^\star = \argmin_{s \in \{\textrm{train}, \textrm{test}, \textrm{self}\}} \min_{\vx'_i \in \bar{\gD}_s} \| g(\vx_i) - g(\vx'_i) \|_2$ \Comment{Find closest set}
    \State $c_{s^\star} \gets c_{s^\star} + 1 / N$ \Comment{Increment counter of closest set}
    \State $\gV_\textrm{train} \gets \gV_\textrm{train} \cup \{ \argmin_{ \vx'_i \in \bar{\gD}_\textrm{train}} \| g(\vx_i) - g(\vx'_i) \|_2 \}$ \Comment{Find closest neighbor in train set}
    \State $\gV_\textrm{test} \gets \gV_\textrm{test} \cup \{ \argmin_{ \vx'_i \in \bar{\gD}_\textrm{test}} \| g(\vx_i) - g(\vx'_i) \|_2 \}$ \Comment{Find closest neighbor in test set}
  \EndFor
  \State $v_\textrm{train} = |\gV_\textrm{train}| / N$, $v_\textrm{test} = |\gV_\textrm{test}| / N$ \Comment{Compute coverage ratio}
\end{algorithmic}
\label{alg:neighbor}
\end{algorithm}

In \autoref{table:similarity_train_test_self_gen_lpips}, we repeat the process used for \autoref{table:similarity_train_test_self_gen} for a subset of its rows by using the pretrained VGG network which measures a Perceptual Image Patch Similarity, also known as LPIPS \citep{zhang2018unreasonable}, instead of the Inception network.
We use the resulting 124,928 concatenated activations and compute their top-100 PCA components.
Overall, the resulting numbers are similar to the ones obtained by the Inception network.
In \autoref{table:similarity_train_test_self_gen_aug}, we use Inception features to compute the complementarity and coverage metrics of various data augmentation schemes.
All augmentation schemes produce samples that are too close to the train set and too far from the test set, which indicates that when they provide samples that could complement the train set, these samples are far from the true distribution.

\begin{figure*}[h]
\centering
\subfigure[Conditional Gaussian]{\includegraphics[width=0.45\textwidth]{images/generated/random.pdf}}
\subfigure[\gls{vdvae}]{\includegraphics[width=0.45\textwidth]{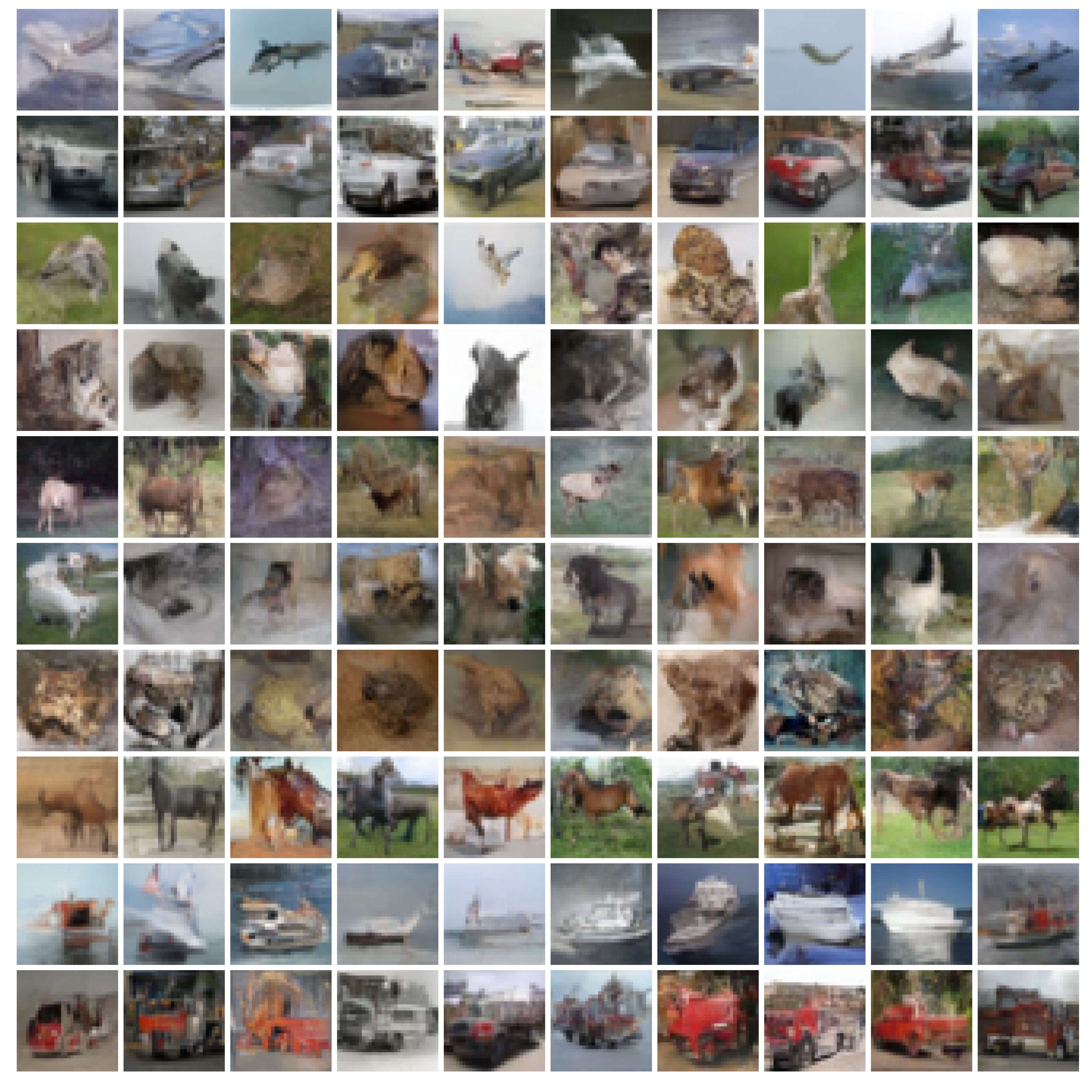}}
\subfigure[BigGAN]{\includegraphics[width=0.45\textwidth]{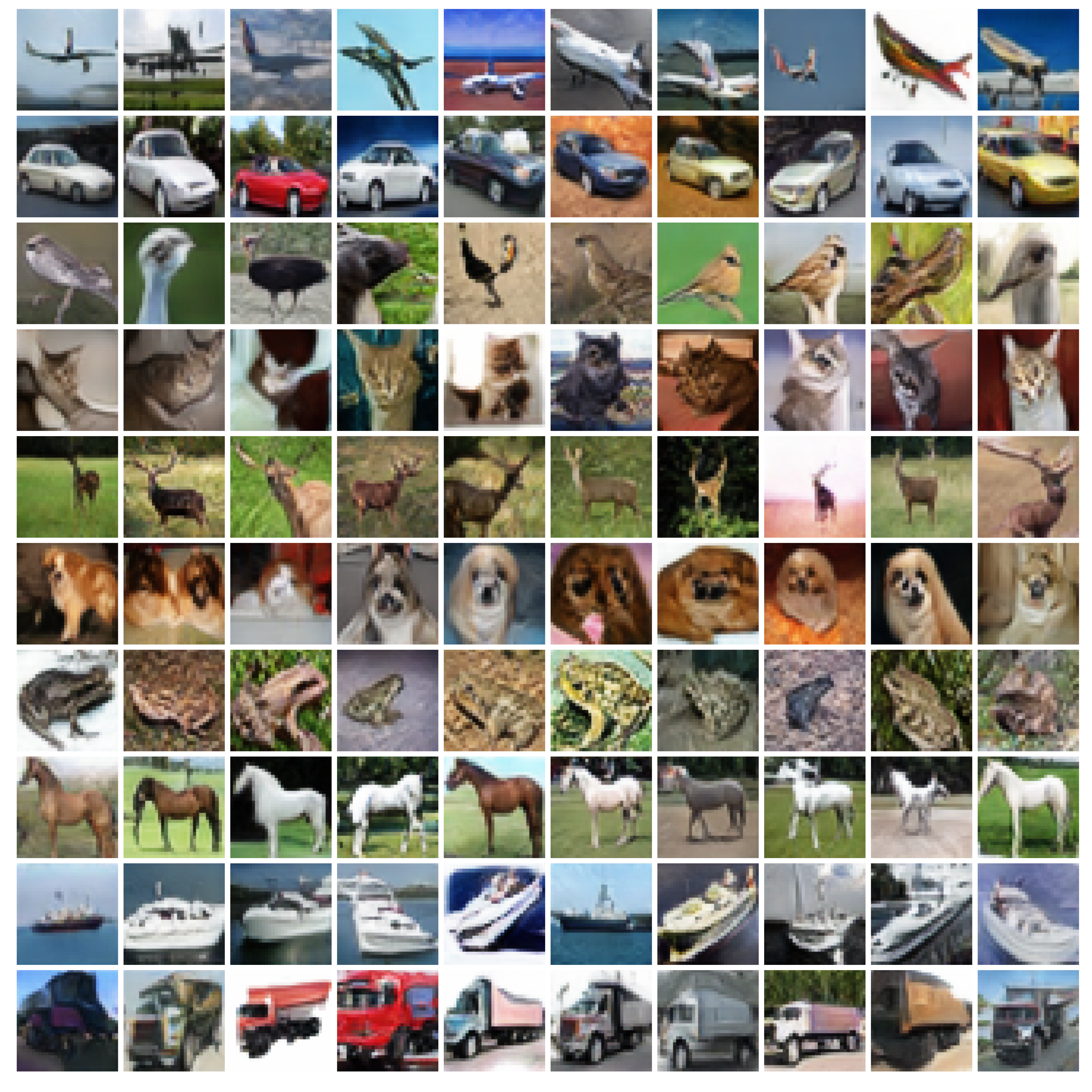}}
\subfigure[StyleGAN]{\includegraphics[width=0.45\textwidth]{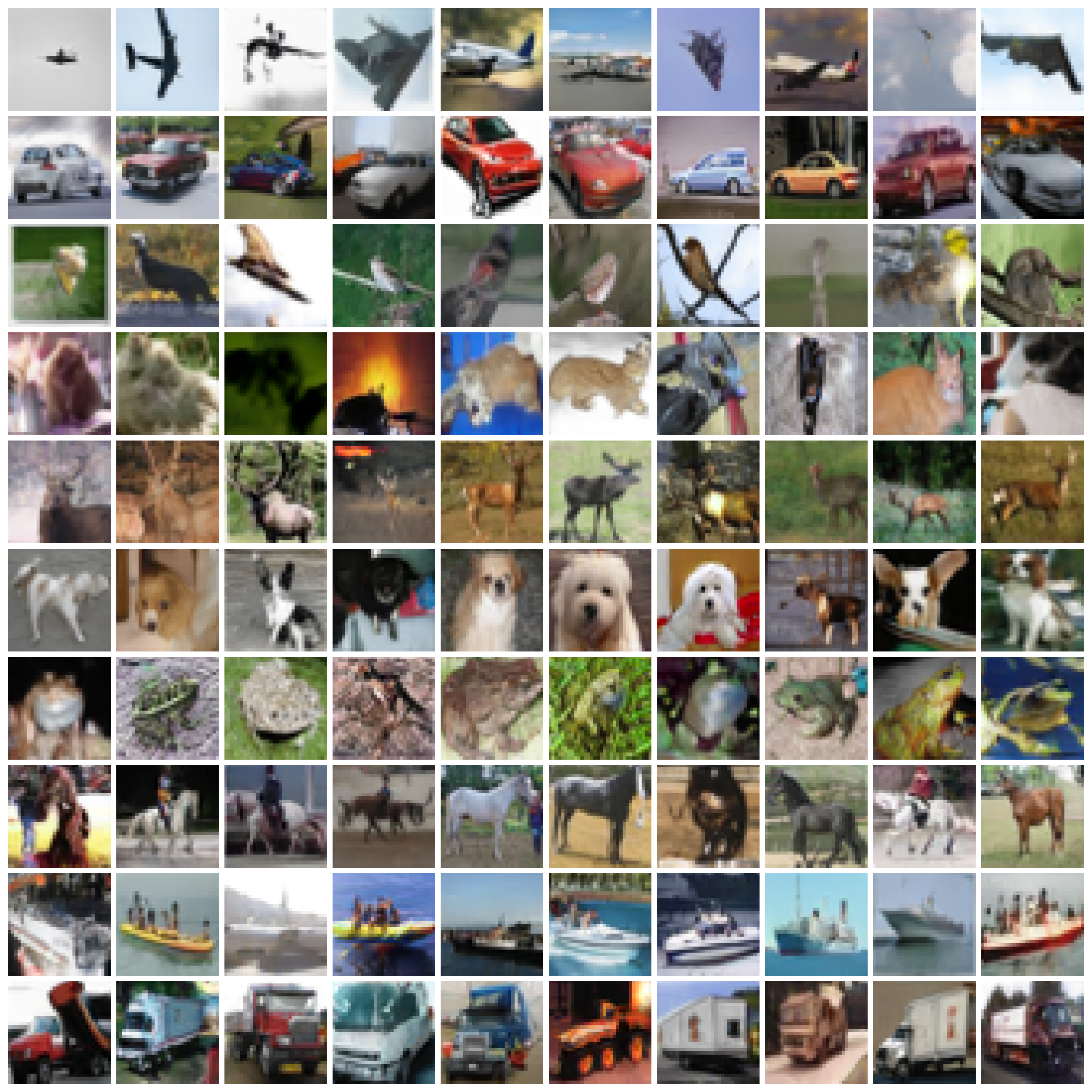}}
\subfigure[\gls{ddpm}]{\includegraphics[width=0.45\textwidth]{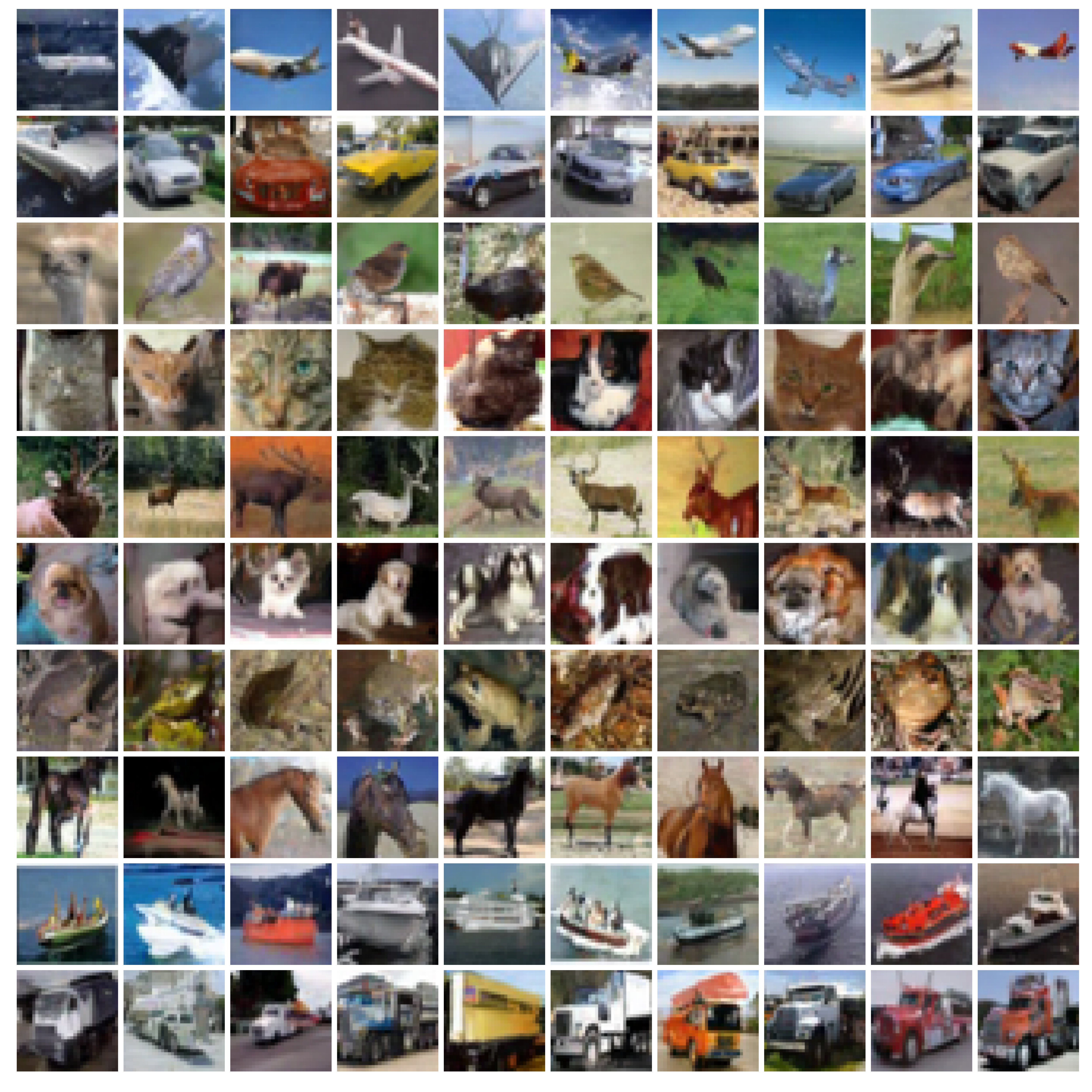}}
\caption{\cifar samples generated by different approaches and used as additional data to train adversarially robust models. Each row correspond to a different class in the following order: airplane, automobile, bird, cat, deer, dog, frog, horse, ship, truck. Each image is assigned a \emph{pseudo-label} using a standard classifier trained on the \cifar train set.}
\label{fig:samples}
\end{figure*}

\begin{table}[t]
\caption{
Complementarity and coverage of augmented and generated samples.
We sample 10K images from the train set and various different generative models.
For each sample in each set, we find its closest neighbor in LPIPS feature space.
To estimate \uline{complementarity}, we report the proportion of samples with a nearest neighbor in either the train set, test set or the sampled set itself.
To estimate \uline{coverage}, we report the proportion of unique neighbors in the train and test set.
We also include the IS and FID computed from 50K samples from each set.
\label{table:similarity_train_test_self_gen_lpips}}
\begin{center}
\resizebox{.85\textwidth}{!}{
\begin{tabular}{l|ccc|cc|cc}
    \hline
    \cellcolor{header} & \multicolumn{3}{c|}{\cellcolor{header} \textsc{Complementarity}} & \multicolumn{2}{c|}{\cellcolor{header} \textsc{Coverage}}  & \multicolumn{2}{c}{\cellcolor{header} \textsc{Inception Metrics}} \TBstrut \\
    \cellcolor{header} \textsc{Setup} & \cellcolor{header} \textsc{Train} & \cellcolor{header} \textsc{Test} & \cellcolor{header} \textsc{Self} & \cellcolor{header} \textsc{Train} & \cellcolor{header} \textsc{Test} & \cellcolor{header} \textsc{Is} $\uparrow$ & \cellcolor{header} \textsc{Fid} $\downarrow$ \TBstrut \\
    \hline
    \emph{mixup}~\citep{zhang2017mixup} & 96.33\% & 0.17\% & 3.50\% & 98.34\% & 41.93\% & $9.33 \pm 0.22$ & 7.71  \TBstrut \\
    \hline
    Class-conditional Gaussian-fit  & 0.80\% & 0.72\% & 98.48\% & 15.80\% & 16.39\% & $3.64 \pm 0.03$ & 117.62 \Tstrut \\
    VDVAE~\citep{child2021vdvae}    & 6.52\% & 5.71\% & 87.77\% & 23.20\% & 23.69\% & $6.88 \pm 0.05$ & 26.44 \\
    BigGAN~\citep{brock2018large}   & 11.69\% & 9.55\% & 78.76\% & 39.29\% & 38.89\% & $9.73 \pm 0.07$ & 13.78 \\
    DDPM~\citep{ho2020denoising}    & 31.20\% & 26.39\% & 42.41\%  & 44.08\% & 43.80\% & $9.50 \pm 0.14$ & 3.15 \Bstrut \\
    \hline
\end{tabular}
}
\end{center}
\end{table}

\begin{table}[t]
\caption{
Complementarity and coverage of augmented samples using the Inception feature space (as done in \autoref{table:similarity_train_test_self_gen}, but for additional data augmentation schemes).
\label{table:similarity_train_test_self_gen_aug}}
\begin{center}
\resizebox{.55\textwidth}{!}{
\begin{tabular}{l|ccc|cc}
    \hline
    \cellcolor{header} & \multicolumn{3}{c|}{\cellcolor{header} \textsc{Complementarity}} & \multicolumn{2}{c}{\cellcolor{header} \textsc{Coverage}} \TBstrut \\
    \cellcolor{header} \textsc{Setup} & \cellcolor{header} \textsc{Train} & \cellcolor{header} \textsc{Test} & \cellcolor{header} \textsc{Self} & \cellcolor{header} \textsc{Train} & \cellcolor{header} \textsc{Test} \TBstrut \\
    \hline
    \emph{mixup}~\citep{zhang2017mixup} & 90.34\% & 3.91\% & 5.75\% & 90.43\% & 45.61\% \Tstrut \\
    \emph{Cutout}~\citep{devries2017improved} & 65.46\% & 3.47\% & 31.07\% & 76.76\% & 45.24\% \\
    \emph{CutMix}~\citep{yun2019cutmix} & 60.30\% & 7.40\% & 32.30\% & 66.05\% & 45.63\% \\
    \emph{AutoAugment}~\citep{cubuk_autoaugment:_2018} & 67.13\% & 6.00\% & 26.87\% & 69.44\% & 45.67\% \\
    \emph{RandAugment}~\citep{cubuk2019randaugment} & 61.23\% & 8.85\% & 29.92\% & 65.51\% & 45.78\% \Bstrut \\
    \hline
\end{tabular}
}
\end{center}
\end{table}

\paragraph{Shortcomings of \gls{fid} and \gls{is}.}

The coverage and complementary metrics from \autoref{table:similarity_train_test_self_gen} provide additional information that is not captured by \gls{fid} and \gls{is}.
In particular, a generative model that memorizes the train set will produce almost perfect \gls{fid} and \gls{is} scores, but will produce a neighbor distribution of 100\% matching the train set, 0\% matching the test set and 0\% matching itself.
This is far from the ideal distribution of $1/3$, $1/3$, $1/3$.
Similarly, a generative model that focuses on a subset of the true distribution can produce high \gls{is}, but low coverage (as exemplified by BigGAN samples). 

\section{Theoretical foundations}

\subsection{Proofs}\label{sec:proof}

Let us consider \autoref{prop:limited} and \autoref{prop:infinite} again.

\setcounter{proposition}{0}
\begin{proposition}[capacity-limited regime]
\autoref{cond:accuracy} and \autoref{cond:gap} are sufficient conditions that allow the sub-optimal parameters $\smash{\hat{\vtheta}^\star}$ to match the performance of the optimal parameters $\smash{\vtheta^\star}$.
\end{proposition}

\begin{proof}
When $\fnr(x) = \fs(x)$ for all $\vx \in \gX$ and $\generateddata$ is such that $\mu(\gW) = \hat{\mu}(\gW)$ for all measurable $\gW \subseteq \mathcal{X}$, \autoref{eq:risk} and \autoref{eq:approx_risk} become identical.
Hence, their solutions achieve the same objective.
\end{proof}

\begin{proposition}[infinite-capacity regime]
\autoref{cond:accuracy}, \autoref{cond:unlikely} and \autoref{cond:coverage} are sufficient conditions that allow the sub-optimal parameters $\smash{\hat{\vtheta}^\star}$ to match the performance of the optimal parameters $\smash{\vtheta^\star}$ when the model $f$ has infinite capacity.
\end{proposition}

\begin{proof}
\autoref{cond:unlikely} guarantees that there are no images with conflicting labels within the perturbation set $\gS$ of a \emph{realistic} image (this extends the non-conflicting labels setup from \autoref{sec:theory_setup})
As such, it is possible to drive the objective from \autoref{eq:approx_risk} to zero.
Since $\fnr(x) = \fs(x)$ for all $\vx \in \gX$ (\autoref{cond:accuracy}) and the generated data covers the true distribution (\autoref{cond:coverage}), the objective obtained from \autoref{eq:approx_risk} can only be zero if the objective to \autoref{eq:risk} is also zero. 
Hence, the solutions of \autoref{eq:risk} and \autoref{eq:approx_risk} achieve the same objective.
\end{proof}

\subsection{Impact of the mixing factor $\alpha$}\label{sec:impact_alpha}

We address here the impact of the mixing factor $\alpha$ used in \autoref{eq:method}.
Ignoring the change of loss, \autoref{eq:method} can be formulated as \autoref{eq:approx_risk} by using the non-robust classifier $f'_\textrm{NR}$ instead of $\fnr$ and using the merged distribution $\smash{\generateddata'}$ instead of $\smash{\generateddata}$, with
\begin{align}
    f'_\textrm{NR}(x) = \begin{cases}
        \fs(\vx) & \textrm{if~} \vx \in \train \\
        \fnr(\vx) & \textrm{otherwise}
    \end{cases}
\end{align}
and with $\smash{\generateddata'}$ such that sampling $\vx \sim \smash{\generateddata'}$ is equivalent to $\vx = [r \leq \alpha] \vx' + [r > \alpha] \vx''$ with $r \sim \mathcal{U}_{[0, 1]}, \vx' \sim \mathcal{U}_\train$ and $\vx'' \sim \smash{\generateddata}$ (where $\mathcal{U}_\sA$ corresponds to the uniform distribution over set $\sA$).
This transformation artificially improves the accuracy of the non-robust classifier and reduces the gap between $\realdata$ and $\smash{\generateddata'}$, thus resulting in better coverage.
Note, however, that while increasing $\alpha$ improves the realism of training samples, it comes at the cost of a reduction in complementarity with the training set.

\section{RobustBench}

For reference, at the time of writing, the top-5 RobustBench (\url{https://robustbench.github.io/}~\citep{croce2020robustbench}) leaderboard entries without and with additional data are listed in \autoref{table:robustbench}. Entries from non-peer reviewed venues were only included if older than 2 months from writing.

\begin{table}[h]
\caption{State of RobustBench leaderboard at the time of writing. We report the clean (without adversarial attacks) accuracy and robust accuracy on \cifar against $\epsilon_\infty = 8/255$.\label{table:robustbench}}
\begin{center}
\resizebox{.5\textwidth}{!}{
\begin{tabular}{ll|cc}
    \hline
    \cellcolor{header} \textsc{Author} & \cellcolor{header} \textsc{Model} & \cellcolor{header} \textsc{Clean} & \cellcolor{header} \textsc{Robust} \TBstrut \\
    \hline
    \hline
    \multicolumn{4}{l}{\cellcolor{subheader} \textsc{Without external data}} \TBstrut \\
    \hline
    \citet{gowal_uncovering_2020} & \wrn-70-16 & 85.29\% & 57.14\% \Tstrut \\
    \citet{gowal_uncovering_2020} & \wrn-34-20 & 85.64\% & 56.82\% \\
    \citet{wu2020adversarial} & \wrn-34-10 & 85.36\% & 56.17\% \\
    \citet{pang2020bag} & \wrn-34-20 & 86.43\% & 54.39\% \\
    \citet{pang_boosting_2020} & \wrn-34-20 & 85.14\% & 53.74\% \Bstrut \\
    \hline
    \hline
    \multicolumn{4}{l}{\cellcolor{subheader} \textsc{With external data}} \TBstrut \\
    \hline
    \citet{gowal_uncovering_2020} & \wrn-70-16 & 91.10\% & 65.87\% \Tstrut \\
    \citet{gowal_uncovering_2020} & \wrn-34-20 & 89.48\% & 62.76\% \\
    \citet{wu2021do} & \wrn-34-15 & 87.67\% & 60.65\% \\
    \citet{wu2020adversarial} & \wrn-28-10 & 88.25\% & 60.04\% \\
    \citet{zhang2021geometryaware} & \wrn-28-10 & 89.36\% & 59.64\% \Bstrut \\
    \hline
\end{tabular}
}
\end{center}
\end{table}

\section{Societal impact}
\label{sec:societal_impact}

Neural networks are being deployed in a wide variety of applications ranging from ranking content on the web~\citep{covington_deep_2016} to autonomous driving~\citep{bojarski_end_2016} via medical diagnostics~\citep{fauw_clinically_2018}.
As such, it is increasingly important to ensure that deployed models are robust and generalize to various input perturbations.
While research on model robustness is welcome for safety critical applications, it is important to note that robustness can sometimes have unforeseen consequences.
In particular, training robust models can lead to models that are overly insensitive to input variations~\citep{tramer_fundamental_2020} and that can increase bias~\citep{chang2020adversarial}.
It is also reported that adversarial robustness may not only be at odds with accuracy~\citep{tsipras_robustness_2018}, but may also be at odds with privacy~\cite{song2019privacy}.

This work also introduces the use of generated data to improve adversarial robustness.
The underlying generative models may leak confidential and private data~\citep{chen2020ganleaks} if they have been trained on a separate dataset.
We protect against this by training generative models from scratch on the same data that is used to train our adversarially robust models.

Finally, our work is the first to match the performance of models trained with additional external data extracted from the ``80 Million Tiny Images'' dataset (\tinyimages) using only the original \cifar dataset.
Since the \tinyimages contains some derogatory terms as categories and offensive images, it has been withdrawn.
As such, we have made our generated datasets available online to allow researchers to avoid the use of \tinyimages.

\end{document}